\documentclass[11pt,letterpaper]{article}
\usepackage[margin=1in]{geometry}
\usepackage{lmodern}
\usepackage[numbers,square]{natbib}
\setcitestyle{citesep={,}}
\usepackage{graphicx}
\usepackage{subcaption}
\usepackage{enumitem}
\usepackage{float}
\usepackage[utf8]{inputenc} %
\usepackage[T1]{fontenc}    %
\usepackage[hidelinks]{hyperref}       %
\usepackage{url}            %
\usepackage{booktabs,tabularx,array}      %
\usepackage{amsfonts}       %
\usepackage{nicefrac}       %
\usepackage{microtype}      %
\usepackage{xcolor}         %
\usepackage{wrapfig}

\usepackage{amsmath}
\usepackage{amssymb}
\usepackage{mathtools}
\usepackage{amsthm}
\usepackage{bbm}

\usepackage{algorithm}
\usepackage{algpseudocode}

\usepackage[capitalize,noabbrev]{cleveref}
\usepackage{caption}
\theoremstyle{plain}
\newtheorem{theorem}{Theorem}[section]

\theoremstyle{definition}

\newtheorem{assumption}[theorem]{Assumption}
\theoremstyle{remark}

\newcommand{\mbb}{\mathbb}
\newcommand{\mbf}{\mathbf}
\newcommand{\mcl}{\mathcal}
\newcommand{\bs}{\boldsymbol}

\newcommand{\T}{\textnormal}
\newcommand{\x}{\mathbf{x}}
\newcommand{\X}{{\mathcal{X}}}

\newcommand{\qpots}{$q\texttt{POTS}$}

\newcommand{\qpotsdoe}{$q\texttt{POTS-DOE}$}

\newcommand{\answerTODO}[1][]{\textcolor{red}{\bf [TODO]}}

\title{Learning What to Evaluate: Correlation-Aware Decoupling for Multiobjective Bayesian Optimization}

\author{%
  Ashwin Renganathan\textsuperscript{1,2}\qquad Peter Bachman\textsuperscript{1}\\[0.5em]
  {\small \textsuperscript{1}Aerospace Engineering}\\
  {\small \textsuperscript{2}Institute of Computational and Data Sciences (ICDS)}\\
  {\small Pennsylvania State University}\\
  {\small University Park, PA 16802}%
}
\date{}

\hypersetup{pdftitle={Learning What to Evaluate: Correlation-Aware Decoupling for Multiobjective Bayesian Optimization}, pdfauthor={Ashwin Renganathan and Peter Bachman}}

\begin{document}

\maketitle

\begin{abstract}
  Multiobjective Bayesian optimization (MOBO) with Gaussian process (GP) surrogates is a sample efficient approach to solving multiobjective optimization problems. In MOBO, a Bayesian decision theoretic acquisition function guides the adaptive selection of new candidate inputs, on which objectives and constraints are evaluated to update the surrogate model sequentially. Existing approaches maintain independent GP models for the objectives and constraints, with new observations evaluating all objectives and constraints in a coupled fashion. However, the objectives and constraints often contain inherent correlations which, if exploited, can enable \emph{decoupled} evaluations where only a subset of them are evaluated at each round. We present a new approach that leverages a multitask GP model to jointly learn all objectives and constraints, and propose a \emph{total correlation} metric that enables identifying an \emph{optimal} subset of objectives and constraints to be evaluated at every round, even under uniform evaluation costs. Theoretically, we show that our acquisition policy is asymptotically consistent despite decoupling and that our proposed decoupled subset selection rule maximizes the expected posterior entropy reduction about unevaluated tasks under mild conditions. Empirically, we show that our approach outperforms coupled and decoupled baselines in the state of the art.
\end{abstract}

\section{Introduction}
\label{sec:intro}

Many scientific and engineering design problems require optimizing several expensive objectives while satisfying expensive constraints. Standard MOBO methods, with GP~\cite{rasmussen2006gaussian} surrogates, usually assume that evaluating a design returns all objectives and constraints together. In many applications, however, separate simulators, tests, operating conditions, or analyses produce these quantities; their evaluations can therefore be ``decoupled''. This creates a scalar-oracle budget: the algorithm may choose both a design and an objective or constraint oracle to query. For instance, aircraft design uses computational aerodynamic simulators to evaluate performance at multiple operating conditions, each evaluated independently~\cite{carlson2025multiobjective}. Neural architecture search (NAS) requires simultaneously optimizing predictive accuracy and model latency on edge devices \cite{janapa2022mlperf}, which are separately tested.

We consider optimizing $K$ objectives subject to $C$ constraints, all independently evaluable:
\begin{equation} 
    \begin{split}
    \max_{\x \in \X}~&\{f_1(\x), \ldots, f_K(\x) \} \\    
    \T{s.t.} ~& c_i(\x) = 0,~i\in \mcl{E} \\
     ~& c_i(\x) \geq 0,~ i \in \mcl{I},
    \end{split}
    \label{eqn:main_problem}
    \end{equation}
where $\x \in \X \subset \mbb{R}^d$ is the design variable, $f_k: \X \rightarrow \mbb{R},~\forall k=1,\ldots,K$ and $c_i : \X \rightarrow \mbb{R}$, the objectives and constraints, respectively, are expensive zeroth-order oracles. $\mcl{E}$ and $\mcl{I}$ index the equality and inequality constraints, respectively, and $C = |\mcl{E} \cup \mcl{I}|$. $\X$ is the domain of $f$ and $c$. We seek a \emph{Pareto} optimal set of solutions that ``Pareto dominates'' all other points. Let $\bs{f(\x)}=[f_1(\x),\ldots,f_K(\x)]^\top$; if $\bs{f(\x)}$ Pareto dominates $\bs{f(\x')}$, then $f_k(\x) \geq f_k(\x'),~\forall k=1,\ldots,K$ and $ \exists k \in [K]$ such that $f_k(\x) > f_k(\x')$. We write Pareto dominance as $\bs{f(\x)} \succ \bs{f(\x')}$. The set $\mcl{Y}^* = \{\bs{f(\x)} ~:~ \nexists \x' \in \X: \bs{f(\x') \succ \bs{f(\x)}} \}$ is the \emph{Pareto frontier}, and the set $\X^* = \{ \x \in \X ~:~ \bs{f(\x)} \in \mcl{Y}^*\}$ is the \emph{Pareto set}. The hypervolume (HV) of a Pareto frontier is the $K-$dimensional Lebesgue measure $\lambda$ of the region dominated by $\mcl{Y}^*$ and bounded below by a reference point $\mbf{r} \in \mbb{R}^K$. This monotonically increasing metric is typically used to assess convergence in MOBO. 

\begin{figure}[ht!]
    \centering
    \includegraphics[width=1\linewidth]{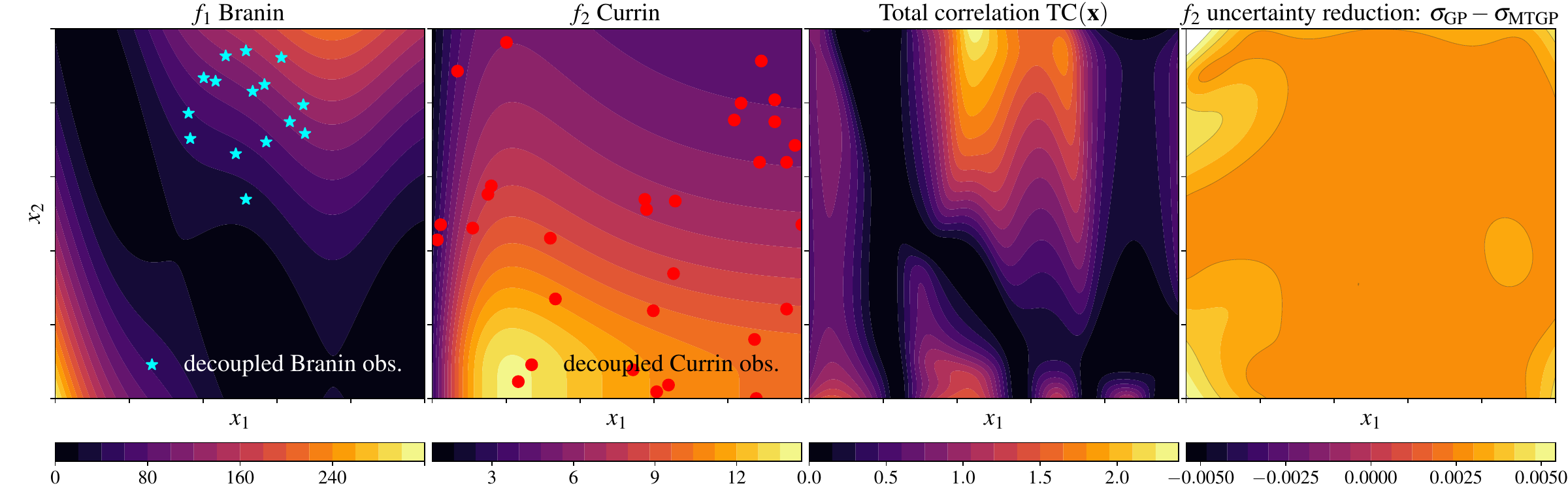}
    \caption{The Branin and Currin test functions (left and middle left panels) and their total correlation $\mathrm{TC}(\x)$ (middle right panel). Knowing the $\mathrm{TC}(\x)$ structure allows us to decouple the two functions' evaluations -- $f_1$ is evaluated at $15$ points (\textcolor{cyan}{$\star$}, left) with \emph{high} $\mathrm{TC}(\x)$ and $f_2$ is evaluated on $30$ points (\textcolor{red}{$\bullet$}, middle left) with \emph{low} $\mathrm{TC}(\x)$.
    Then, we fit two independent GPs using all $45$ points for both $f_1$ and $f_2$ ($90$ evaluations total) and one multitask GP (MTGP) with only the $45$ \emph{decoupled} evaluations.
    The right panel shows the MTGP's reduction of posterior uncertainty in $f_2$. Thus, appropriately chosen observation sites allow MTGP to transfer information between tasks, which can improve sample efficiency in MOBO.}
    \label{fig:bc_tc_illustration}
\end{figure}

In MOBO, a decision-theoretic rule (the acquisition function) typically selects the next evaluation point \citep{jones1998efficient,shahriari2016taking,frazier2018tutorial, knowles2006parego}.
When objectives or constraints are statistically correlated, a query to one oracle can reduce posterior uncertainty about others. Evaluating every oracle at every selected design may then be redundant. The key question is therefore not only where to evaluate, but also which subset of objectives and constraints to evaluate at each design.
Multitask Gaussian process (MTGP) models~\cite{bonilla2008multi} offer an effective way to jointly learn correlated tasks.
\Cref{fig:bc_tc_illustration} illustrates the mechanism we exploit: regions with high ``total correlation'' (defined in \Cref{sec:method}) indicate where observing one task can inform another. In the Branin-Currin example, a multitask posterior trained on selectively decoupled observations has lower uncertainty for one objective than independent, ``single task'', models unable to transfer information across tasks.

Traditionally, independent GPs emulate objectives and constraints in MOBO. For instance, the expected hypervolume improvement (EHVI)~\citep{emmerich2008exi}, its stochastic extension~\citep{daulton2020differentiable}, and predictive entropy search (PESMO)~\citep{garrido2016pesmoc} all assume independent GP priors across objectives and constraints. This independence assumption permits simple product factorizations yielding closed-form acquisition functions. 
However, MTGPs complicate acquisition function construction because cross-output correlations prevent product factorizations possible with independent GPs. Standard acquisitions thus become high-dimensional integrals over correlated Gaussians (truncated Gaussians under constraints), preventing their widespread use in MOBO. Despite this limitation, MTGPs have been used in single-objective BO to show that sharing structure across tasks can substantially reduce the evaluations needed for good solutions \citep{swersky2013multitaskbo}. In this work, we combine MTGPs with a total-correlation-based metric to propose an optimal way to decouple objective and constraint evaluations in MOBO. A constrained MOBO method is needed to exploit inter-task correlation to decouple objective and constraint evaluations \emph{even under uniform oracle evaluation costs}, while avoiding intractable MTGP acquisition computations. 
    
\paragraph{Related work.} Recently, \citet{renganathan2025qpots} extended Thompson sampling~\cite{thompson1933likelihood} to constrained MOBO as Pareto optimal Thompson sampling (qPOTS). qPOTS is attractive here because it replaces acquisition integration over correlated predictive distributions with optimization over draws of posterior sample paths. Drawing joint sample paths from an MTGP is straightforward, whereas adapting EHVI or PESMO style acquisitions to correlated objective/constraint posteriors leads to difficult high-dimensional integrals. Therefore, we build on qPOTS~\cite{renganathan2025qpots} in this work.

MOBO with decoupled evaluations has been addressed before, albeit rarely. These include the entropy-based approaches Predictive Entropy Search
(PESMO)~\cite{garrido2016pesmoc}, Pareto Frontier Entropy Search (PFES)~\cite{suzuki2020multi}, and more recently the Joint Entropy Search (JES)~\cite{tu2022joint}. Other approaches include the hypervolume knowledge gradient (HVKG)~\cite{daulton_hypervolume_nodate} and another variant of knowledge gradient \cite{buckingham2025knowledge}.
Outside of these two categories, FlexiBO~\cite{iqbal2023flexibo} uses the uncertainty reduction in the hypervolume to accommodate decoupled evaluations. However, all of these methods share the following limiting characteristics. (1) They require the cost of evaluating each objective to be different for the decoupling to have a value proposition. Under uniform evaluation costs, they may not provide benefit. (2) They do not consider the decoupling of constraint evaluations. (3) They do not exploit the inherent correlation between objectives; instead, decoupling is achieved as a consequence of cost weighting. (4) They are not directly compatible with MTGPs, which would necessitate appropriate modifications to their acquisition function computation.
On the other hand, our proposed approach overcomes all of these limitations, while enabling decoupling \emph{under uniform evaluation costs} -- this crucial benefit broadens decoupling to scenarios where cost is not a factor or unknown. Our take is that, as long as objectives and constraints are correlated and decoupled evaluations are feasible, it should be exploited for better sample efficiency in MOBO regardless of any difference in evaluation costs.
 Our contributions are summarized as follows.
\begin{enumerate}[leftmargin=*]
    \item {\bf Method.} We introduce \qpotsdoe: ``batch Pareto optimal Thompson sampling with decoupled oracle evaluations'' for constrained MOBO with partial objective and constraint observations under uniform scalar-oracle costs.
    \item {\bf Decoupling rule.} 
    We propose a two-stage decoupling mechanism: a total correlation (\Cref{sec:total_correlation}) gate to decide whether decoupling is warranted at a candidate design, followed by a Gaussian mutual-information criterion to choose the optimal subset of objectives and constraints to evaluate. 
    \item {\bf Theory.} We show that despite decoupling our acquisition policy is asymptotically consistent, and that the chosen subset of decoupled evaluations is optimal.
    \item {\bf Empirics.} We demonstrate improved hypervolume performance and evaluation efficiency on synthetic and real-world constrained MOBO benchmarks, including ablations over the TC threshold and batch size.
\end{enumerate}
The rest of the paper is organized as follows. In \Cref{sec:preliminaries} we briefly review preliminaries. Our overall method is summarized in \Cref{sec:method}. Our theoretical and empirical results are presented in \Cref{sec:theory} and \Cref{sec:experiments}, respectively. We provide concluding remarks in \Cref{sec:conclusions}. Our software implementation is openly available at \url{https://github.com/csdlpsu/qpots}.

\section{Preliminaries}
\label{sec:preliminaries}
\paragraph{Multitask Gaussian process models.} A review of ``single task'' GP models is provided in \Cref{sec:gp_intro}.
For MTGP models, the output at each input $\mathbf{x}\in\mathbb{R}^d$ is $K$-dimensional. We define the vector-valued process
\begin{equation*}
  \mathbf{Y}(\mathbf{x})
  \;=\;
  \big[Y_1(\mathbf{x}),\,\ldots,\,Y_K(\mathbf{x})\big]^\top
  \in \mathbb{R}^K,
\end{equation*}
where $Y_k(\mathbf{x})$ denotes the latent function for task $k\in\{1,\ldots,K\}$. An MTGP assumes that any finite collection of stacked outputs is jointly Gaussian \cite{bonilla2008multi,alvarez2012kernels}. We write $\mathbf{Y}(\mathbf{x})
  \sim
  \mathrm{MTGP}\!\big(\boldsymbol{\mu}(\mathbf{x}),\, \mathbf{K}(\mathbf{x},\mathbf{x}')\big),$
where $\boldsymbol{\mu}(\mathbf{x})\in\mathbb{R}^K$ is the mean vector and $\mathbf{K}(\mathbf{x},\mathbf{x}')
  \;=\;
  \mathrm{cov}\!\big(\mathbf{Y}(\mathbf{x}),\mathbf{Y}(\mathbf{x}')\big)
  \in \mathbb{R}^{K\times K}$
is a \emph{matrix-valued} kernel that encodes both input similarity and inter-task correlations \cite{alvarez2012kernels}.
Given inputs $\mathbf{X}=[\mathbf{x}_1,\ldots,\mathbf{x}_n]^\top$, define the stacked output vector
\begin{equation*}
  \mathbf{y}
  \;=\;
  \big[\mathbf{Y}(\mathbf{x}_1)^\top,\ldots,\mathbf{Y}(\mathbf{x}_n)^\top\big]^\top
  \in \mathbb{R}^{nK}.
\end{equation*}
Then $\mathbf{y} \sim \mathcal{N}(\boldsymbol{\mu}, \Sigma),$
where $\mathbf{m}=[\boldsymbol{\mu}(\mathbf{x}_1)^\top,\ldots,\boldsymbol{\mu}(\mathbf{x}_n)^\top]^\top$ and $\Sigma\in\mathbb{R}^{nK\times nK}$ is a block covariance matrix with $K\times K$ blocks $ \Sigma_{ij} \;=\; \mathbf{K}(\mathbf{x}_i,\mathbf{x}_j),
  ~ i,j\in\{1,\ldots,n\}.$
A common observation model is additive Gaussian noise per input: $\tilde{\mathbf{y}}_i = \mathbf{Y}(\mathbf{x}_i) + \boldsymbol{\varepsilon}_i,~
  \boldsymbol{\varepsilon}_i \sim \mathcal{N}(\mathbf{0}, \Sigma_\varepsilon),$
with $\Sigma_\varepsilon\in\mathbb{R}^{K\times K}$ often chosen diagonal (independent task noise) or full (correlated noise). Stacking all observations gives
\begin{equation*}
  \tilde{\mathbf{y}} \sim \mathcal{N}\!\big(\mathbf{m},\, \Sigma + \mathbf{I}_n \otimes \Sigma_\varepsilon\big),
\end{equation*}
where $\otimes$ denotes the Kronecker product.
For a test input $\mathbf{x}_*$, define $  \Sigma_{X*} =
  \big[\mathbf{K}(\mathbf{x}_i,\mathbf{x}_*)\big]_{i=1}^n
  \in \mathbb{R}^{nK\times K},
  ~
  \Sigma_{*X} = \Sigma_{X*}^\top,~
  \Sigma_{**} = \mathbf{K}(\mathbf{x}_*,\mathbf{x}_*) \in \mathbb{R}^{K\times K}.$
Conditioning the joint Gaussian on observations yields \cite{bonilla2008multi,alvarez2012kernels}:
\begin{equation}
  \begin{aligned}
  \mathbf{Y}(\mathbf{x}_*) \mid \mathbf{X}, \tilde{\mathbf{y}}
  \;\sim\; \mathcal{N}\!\Big(
    &\mathbf{m}(\mathbf{x}_*) + \Sigma_{*X}\big(\Sigma + \mathbf{I}_n \otimes \Sigma_\varepsilon\big)^{-1}\big(\tilde{\mathbf{y}}-\mathbf{m}\big),\\
    &\Sigma_{**} - \Sigma_{*X}\big(\Sigma + \mathbf{I}_n \otimes \Sigma_\varepsilon\big)^{-1}\Sigma_{X*}
  \Big).
  \end{aligned}
  \label{eq:mtgp_posterior}
\end{equation}

A central modeling choice is the form of the matrix-valued kernel $\mathbf{K}(\mathbf{x},\mathbf{x}')$.
A widely used parameterization is the \emph{intrinsic coregionalization model} (ICM), $\mathbf{K}(\mathbf{x},\mathbf{x}') = \mathbf{B}\, k(\mathbf{x},\mathbf{x}'),$
where $k(\cdot,\cdot)$ is a scalar kernel over inputs and $\mathbf{B}\in\mathbb{R}^{K\times K}$ is a positive semidefinite \emph{coregionalization} matrix capturing task relatedness \cite{bonilla2008multi,alvarez2012kernels}. Under the ICM kernel, the stacked covariance takes the Kronecker form $\Sigma = \Sigma_\x \otimes \mathbf{B},~
  (\Sigma_\x)_{ij} = k(\mathbf{x}_i,\mathbf{x}_j),$
which can be exploited computationally when $n$ and $K$ are large.
More generally, the \emph{linear model of coregionalization} (LMC) uses a sum of separable components,
allowing multiple latent correlation structures across tasks and inputs \cite{alvarez2012kernels,teh2005semiparametric}. This flexibility often improves transfer across related tasks by sharing statistical strength while preserving task-specific variation.

\section{\texorpdfstring{$q\texttt{POTS-DOE}$}{qPOTS-DOE}: Batch Pareto optimal Thompson sampling with decoupled oracle evaluations}
\label{sec:method}
Instead of fitting $K$ independent posterior
GP models in \cite{renganathan2025qpots}, we fit a single $K$-task MTGP model for all objectives. We directly address the constrained setting; but review the unconstrained setting in \Cref{sec:unconstrained_qpots}.

In addition to $K$ objectives, let there be a total of $C$ constraints $\{c_1,\ldots,c_C\}$, where
$\mcl{I}\cup \mcl{E}=[C]$. We
fit a single $(K+C)$-task multitask GP model for all objectives
and constraints. Let
\[
\mathbf{Y}_{1:K+C}(\x)
\equiv
\begin{bmatrix}
Y_1(\x) & \cdots & Y_K(\x) &
Y_{K+1}(\x) & \cdots & Y_{K+C}(\x)
\end{bmatrix}^{\top},
\]
where, for each $i\in[C]$, $Y_{K+i}$ denotes the posterior task
corresponding to constraint $c_i$. The joint multitask posterior, with hyperparameters $\Omega$, is
\[
\mathbf{Y}_{1:K+C}(\cdot)\mid D_n^{1:K+C},\Omega
\sim
\mathrm{MTGP}\!\left(
\boldsymbol{\mu}_n(\cdot),
\boldsymbol{\Sigma}_n(\cdot,\cdot)
\right),
\]
with
\[
\left[\boldsymbol{\Sigma}_n(\x,\x')\right]_{ij}
=
\operatorname{cov}
\left(
Y_i(\x),Y_j(\x') \mid D_n^{1:K+C},\Omega
\right),
\qquad i,j\in[K+C].
\]
Hence, the objectives and constraints are learned jointly, and the
posterior cross-covariances between objective tasks and constraint tasks
are retained.
For a posterior sample path, define the sampled feasible region as
\[
\mathcal{X}^{\mathrm{feas}}_{\mathbf{Y}}
=
\left\{
\x\in\mathcal{X}:
\mathbbm{1}\!\left\{
\bigcap_{i\in \mcl E} Y_{K+i}(\x)=0
\right\}
\mathbbm{1}\!\left\{
\bigcap_{i\in \mcl I} Y_{K+i}(\x)\geq 0
\right\}
=1
\right\}.
\]
Then, we choose points according to
\[
p_{X^\ast}(\x)
=
\int
\delta
\left(
\x-
\operatorname*{arg\,max}_{\x\in\mathcal{X}^{\mathrm{feas}}_{\mathbf{Y}}}
\{Y_1(\x),\ldots,Y_K(\x)\}
\right)
p\!\left(\mathbf{Y}_{1:K+C}\mid \mcl{D}_n^{1:K+C}\right)
\,d\mathbf{Y}_{1:K+C},
\]
where $\delta$ is the Dirac's delta function and $d\mathbf{Y}_{1:K+C}
\equiv
dY_1\cdots dY_K\,dY_{K+1}\cdots dY_{K+C}.$
Equivalently, we draw one joint sample path from the $(K+C)$-task
multitask GP posterior,
\[
\mathbf{Y}_{1:K+C}(\cdot,\omega)
=
\boldsymbol{\mu}_n(\cdot)
+
\boldsymbol{\Sigma}_n^{1/2}(\cdot)\,Z(\omega),
\qquad
Z(\omega)\sim\mathcal{N}(0,I),
\]
filter the sampled design space by the same feasibility indicators, and
solve
\[
X^\ast
=
\operatorname*{arg\,max}_{x\in\mathcal{X}^{\mathrm{feas}}_{\omega}}
\{Y_1(\x,\omega),\ldots,Y_K(\x,\omega)\},
\]
where
\[
\mathcal{X}^{\mathrm{feas}}_{\omega}
=
\left\{
\x\in\mathcal{X}:
\mathbbm{1}\!\left\{
\bigcap_{i\in E} Y_{K+i}(\x,\omega)=0
\right\}
\mathbbm{1}\!\left\{
\bigcap_{i\in I} Y_{K+i}(\x,\omega)\geq 0
\right\}
=1
\right\}.
\]
This cheap multiobjective optimization problem is solved via evolutionary approaches e.g., NSGA-II~\cite{deb2002fast}. 
Let
$\gamma(\cdot,\cdot):\mathbb{R}^d\times\mathbb{R}^d\rightarrow
\mathbb{R}_{+}$ denote the Euclidean distance between two points in
$\mathcal{X}^{\mathrm{feas}}_{\omega}$.
Then, a batch of $q$ acquisitions is chosen per
\begin{equation}
    \begin{aligned}
\x_{n+1}
&=
\operatorname*{arg\,max}_{\x^\ast\in X^\ast}
\min_{\x_i\in \mbf{X}}
\gamma(\x^\ast,\x_i), \\
\x_{n+2}
&=
\operatorname*{arg\,max}_{\x^\ast\in X^\ast}
\min_{\x_i\in \mbf{X}\cup\{\x_{n+1}\}}
\gamma(\x^\ast,\x_i), \\
&\hspace{0.25in}\vdots \\
\x_{n+q}
&=
\operatorname*{arg\,max}_{\x^\ast\in X^\ast}
\min_{\x_i\in \mbf{X}\cup\{\x_{n+1},\ldots,\x_{n+q-1}\}}
\gamma(\x^\ast,\x_i).
\end{aligned}
\label{eqn:maximin_}
\end{equation}
Once a batch of $q$ points is chosen per \eqref{eqn:maximin_}, we then decide if evaluations must be decoupled, followed by choosing the optimal subset of them to evaluate -- this is presented next.

\subsection{The total correlation metric \texorpdfstring{$\mathrm{TC}(\x)$}{TC(x)} thresholding for decoupling}
\label{sec:total_correlation}
A key advantage of MTGPs is their ability to model {statistical dependence} across tasks.
After conditioning on data $\mathcal{D}$, the predictive distribution at a fixed input $\mathbf{x}$ is Gaussian,
$\mathbf{Y}(\mathbf{x}) \mid \mathcal{D}
  \;\sim\;
  \mathcal{N}\!\big(\boldsymbol{\mu}(\mathbf{x}),\, \Sigma(\mathbf{x})\big),
  ~
  \Sigma(\mathbf{x}) \in \mathbb{R}^{K\times K},$
where $\Sigma(\mathbf{x})$ is the \emph{predictive} covariance matrix across tasks at $\mathbf{x}$. Then, a standard normalized measure of pairwise coupling is the (posterior) correlation coefficient
\[
  \rho_{k\ell}(\mathbf{x})
  \;=\;
  \frac{\Sigma_{k\ell}(\mathbf{x})}{\sqrt{\Sigma_{kk}(\mathbf{x})\,\Sigma_{\ell\ell}(\mathbf{x})}},
  \qquad k\neq \ell,
  \label{eq:pairwise_corr}
\]
and the full correlation matrix can be written compactly as
\begin{equation}
  \mathbf{R}(\mathbf{x})
  \;=\;
  \mathbf{D}(\mathbf{x})^{-\tfrac{1}{2}}\,\Sigma(\mathbf{x})\,\mathbf{D}(\mathbf{x})^{-\tfrac{1}{2}},
  \quad
  \mathbf{D}(\mathbf{x})=\mathrm{diag}\!\big(\Sigma(\mathbf{x})\big).
  \label{eq:corr_matrix}
\end{equation}
While pairwise correlations $\rho_{k\ell}(\mathbf{x})$ are intuitive, they do not provide a single scalar summary of the \emph{overall} dependence among all $K$ tasks at $\mathbf{x}$ -- we overcome this with a ``total correlation'' metric (also called \emph{multi-information})
\cite{watanabe1960information,cover2006elements}, defined as
\begin{equation}
  \mathrm{TC}(\mathbf{x})
  \;\triangleq\;
  D_{\mathrm{KL}}\!\Big(
    p\big(\mathbf{Y}(\mathbf{x})\mid \mathcal{D}\big)
    \,\Big\|\,
    \prod_{k=1}^K p\big(Y_k(\mathbf{x})\mid \mathcal{D}\big)
  \Big),
  \label{eq:tc_def_kl}
\end{equation}
which measures how far the joint predictive distribution across tasks deviates from the product of its (task-wise) marginals. Equivalently, in entropy form,
$ \mathrm{TC}(\mathbf{x})
  \;=\;
  \sum_{k=1}^K H\!\big(Y_k(\mathbf{x})\mid \mathcal{D}\big)
  \;-\;
  H\!\big(\mathbf{Y}(\mathbf{x})\mid \mathcal{D}\big),$
so $\mathrm{TC}(\mathbf{x})\ge 0$, and $\mathrm{TC}(\mathbf{x})=0$ if and only if the tasks are independent under the predictive distribution at $\mathbf{x}$ \cite{cover2006elements}.
Because $\mbf{Y}$ is multivariate Gaussian, $\mathrm{TC}(\mathbf{x})$ has a simple closed form in terms of $\Sigma(\mathbf{x})$:
\begin{equation*}
  \mathrm{TC}(\mathbf{x})
  \;=\;
  \frac{1}{2}\log\!\frac{\big|\mathrm{diag}\big(\Sigma(\mathbf{x})\big)\big|}{\big|\Sigma(\mathbf{x})\big|}
  \;=\;
  \frac{1}{2}\log\!\frac{\prod_{k=1}^K \Sigma_{kk}(\mathbf{x})}{\big|\Sigma(\mathbf{x})\big|} = -\frac{1}{2}\log \big|\mathbf{R}(\mathbf{x})\big|,
  \label{eq:tc_gaussian}
\end{equation*}
where the last equality uses the identity $|\Sigma(\mathbf{x})| = |\mathbf{D}(\mathbf{x})|\,|\mathbf{R}(\mathbf{x})|$.
This highlights that (for Gaussian predictive distributions) total correlation depends only on the correlation structure, not on the marginal scales.

Finally, we use $\mathrm{TC}(\x) \geq \tau$ as a gate to decide if decoupling is warranted at a candidate $\x$ or not, where $\tau > 0$ is a threshold the user must specify. As $\tau \rightarrow \infty$, \qpotsdoe~reduces to qPOTS~\cite{renganathan2025qpots}. Note that $\tau$ represents the amount of information (in nats) that can be transferred at $\x$ between tasks; thus when $\mathrm{TC}(\x) \geq \tau$, $\mathrm{TC}(\x)$ is an upper bound on the maximum possible information transfer from any decoupled subset. 
In practice we observe (and as shown in the ablation in \Cref{fig:ablations}) that choosing $\tau \in [10^{-6}, 10^{0}]$ has negligible impact on performance. Once a candidate $\x$ is chosen, and the total correlation gate ($\mathrm{TC}(\x) > \tau$) passes, an optimal subset of objectives and constraints for decoupled evaluations is chosen by solving an inner optimization problem -- this is described next.

\subsection{Optimal subset for decoupled evaluation}
Without loss of generality, we introduce an {inner} combinatorial optimization problem that selects a subset of $M$ objectives out of $K$.\footnote{In the constrained setting, we would select from $K+C$ instead.} Let $S \subseteq \mathcal{K}\triangleq\{1,\ldots,K\}$ and $\bar{S} \triangleq \mathcal{K}\setminus S,$
so that the selected objectives $\mathbf{Y}_S(\mathbf{x})$ are maximally informative about the remaining objectives $\mathbf{Y}_{\bar{S}}(\mathbf{x})$.
With a cardinality budget $|S|=M$ (that is, we can evaluate $M$ objectives per candidate $\mathbf{x}$), we define the inner problem as
\begin{equation}
  S^\star(\mathbf{x})
  \;\in\;
  \arg\max_{S \subseteq \mathcal{K} \,:\, |S|=M}
  I\!\big(\mathbf{Y}_S(\mathbf{x}); \mathbf{Y}_{\bar{S}}(\mathbf{x}) \mid \mathcal{D}\big).
  \label{eqn:inner_subset_mi}
\end{equation}
This criterion optimally selects $M$ objectives that are (i) strongly coupled with the rest while (ii) avoiding redundancy among selected objectives, since redundant tasks provide little incremental information. We later show how to determine the cardinality budget $M$.
Using the identity $I(A;B)=H(A)+H(B)-H(A,B)$ and the Gaussian differential entropy
$H(\mathbf{Z})=\tfrac{1}{2}\log\!\big((2\pi e)^m|\Sigma_{\mathbf{Z}}|\big)$ for $\mathbf{Z}\in\mathbb{R}^m$
\cite{cover2006elements}, we obtain
\begin{align}
  I\!\big(\mathbf{Y}_S(\mathbf{x}); \mathbf{Y}_{\bar{S}}(\mathbf{x}) \mid \mathcal{D}\big)
  &=
  H\!\big(\mathbf{Y}_S(\mathbf{x})\mid\mathcal{D}\big)
  + H\!\big(\mathbf{Y}_{\bar{S}}(\mathbf{x})\mid\mathcal{D}\big)
  - H\!\big(\mathbf{Y}(\mathbf{x})\mid\mathcal{D}\big)
  \nonumber\\
  &=
  \frac{1}{2}\log
  \frac{
    \big|\Sigma_{SS}(\mathbf{x})\big|\;
    \big|\Sigma_{\bar{S}\bar{S}}(\mathbf{x})\big|
  }{
    \big|\Sigma(\mathbf{x})\big|
  }.
  \label{eq:mi_det_ratio}
\end{align}
Here $\Sigma_{SS}(\mathbf{x})$ and $\Sigma_{\bar{S}\bar{S}}(\mathbf{x})$ are principal submatrices of $\Sigma(\mathbf{x})$ indexed by $S$ and $\bar{S}$, respectively.
Since $\log|\Sigma(\mathbf{x})|$ does not depend on $S$, maximizing \eqref{eq:mi_det_ratio} is equivalent to
\begin{equation*}
  S^\star(\mathbf{x})
  \in
  \arg\max_{S \subseteq \mathcal{K}:\,|S|=M}
  \Big\{
    \log\big|\Sigma_{SS}(\mathbf{x})\big|
    +
    \log\big|\Sigma_{\bar{S}\bar{S}}(\mathbf{x})\big|
  \Big\},
  \label{eq:inner_subset_logdet}
\end{equation*}
which can be interpreted as finding a partition (of fixed size) that makes both sides ``internally uncertain'' while being as inter-dependent as possible. In \Cref{thm:finite_sample_information_transfer}, we show that this rule selects the \emph{best} $M$ out of $K$ oracles to evaluate, in terms of the one-step expected entropy reduction.

\paragraph{Relationship to total correlation.}
Define the within-subset total correlations
$\mathrm{TC}_S(\mathbf{x})$ and $\mathrm{TC}_{\bar{S}}(\mathbf{x})$ analogously but restricted to objectives in $S$ and $\bar{S}$. Using $\mathrm{TC}= \sum_k H(Y_k)-H(\mathbf{Y})$ and elementary cancellations,
\begin{equation*}
  I\!\big(\mathbf{Y}_S(\mathbf{x}); \mathbf{Y}_{\bar{S}}(\mathbf{x}) \mid \mathcal{D}\big)
  =
  \mathrm{TC}(\mathbf{x}) - \mathrm{TC}_S(\mathbf{x}) - \mathrm{TC}_{\bar{S}}(\mathbf{x}),
  \label{eq:mi_tc_decomp}
\end{equation*}
so maximizing cross-subset mutual information (\Cref{eqn:inner_subset_mi}) favors partitions where the \emph{between-group} dependence is large relative to the \emph{within-group} dependence.

The inner problem \eqref{eqn:inner_subset_mi} is combinatorial: it requires first choosing a subset size $M$, and then choosing the optimal subset among $\binom{K}{M}$ combinations. Related mutual-information subset selection problems are known to be NP-hard in GP settings \cite{krause2008near}.
Accordingly, we consider two practical approaches:

\noindent\textbf{(1) Exact maximization (small $K$).}
When $K$ is small (typical in multiobjective optimization), we can evaluate \eqref{eq:mi_det_ratio} for all subsets of size $M$, for any $M$, and select the maximizer.
Each evaluation requires computing log-determinants of principal submatrices; this can be done stably via Cholesky factorization. If $\Sigma_{SS}(\mathbf{x}) = \mathbf{L}\mathbf{L}^\top$, $  \log\big|\Sigma_{SS}(\mathbf{x})\big|
  =
  2\sum_{i}\log L_{ii},$
and similarly for $\Sigma_{\bar{S}\bar{S}}(\mathbf{x})$.

\smallskip
\noindent\textbf{(2) Approximate maximization (moderate/large $K$).}
Define the set function
\begin{equation*}
  f_{\mathbf{x}}(S)
  \;\triangleq\;
  I\!\big(\mathbf{Y}_S(\mathbf{x}); \mathbf{Y}_{\bar{S}}(\mathbf{x}) \mid \mathcal{D}\big).
  \label{eq:set_function_fx}
\end{equation*}
Because $f_{\mathbf{x}}(S)=H(\mathbf{Y}_S)+H(\mathbf{Y}_{\bar{S}})-H(\mathbf{Y})$ and entropy is submodular,
$f_{\mathbf{x}}$ is a (typically non-monotone) \emph{symmetric submodular} function of $S$ \cite{cover2006elements,krause2008near}.
This enables the use of standard submodular maximization heuristics and approximation algorithms
\cite{feige2011maximizing,buchbinder2015approximation}. A simple and effective heuristic under the cardinality constraint $|S|=M$ is a \emph{greedy} construction followed by \emph{swap} local search (starting with $S=\varnothing$):
\begin{equation}
    \begin{split}
  \Delta_{\mathbf{x}}(s\mid S) &\triangleq f_{\mathbf{x}}(S\cup\{s\}) - f_{\mathbf{x}}(S),\\
  s^\star &\in \arg\max_{s\in \mathcal{K}\setminus S} \Delta_{\mathbf{x}}(s\mid S),  
\end{split}
\label{eqn:subset_largek}
\end{equation}
iterated until $|S|=M$, and then refined by swapping $i\in S$ with $j\in\bar{S}$ whenever
$f_{\mathbf{x}}(S\setminus\{i\}\cup\{j\})$ improves. Then, we repeat \Cref{eqn:subset_largek} for $M=1,\ldots,K-1$ to pick M that gives the smallest $f_\x(S)$.
More sophisticated choices (e.g., randomized methods with constant-factor guarantees for non-monotone submodular maximization) can be used when $K$ is large \cite{feige2011maximizing,buchbinder2015approximation}. Our experiments use the exact maximization approach, but we perform an ablation on exact vs approximate maximization in \Cref{fig:ablation_subset_selection}, showing that under moderate $K+C$, the accuracy difference is not substantial.

If objective evaluations at $\mathbf{x}$ are noisy, one can compute the mutual information between the \emph{noisy} observed subset and the \emph{latent} unobserved objectives by replacing the observed covariance block with
$\Sigma_{SS}(\mathbf{x})+\Sigma_{\varepsilon,SS}$ in \eqref{eq:mi_det_ratio}, where $\Sigma_{\varepsilon,SS}$ is the noise covariance for the selected objectives. This makes the subset selection explicitly account for heteroscedastic or task-dependent noise. All the steps involved are summarized in \Cref{alg:qpots-doe}.

\section{Theoretical results}
\label{sec:theory}
We prove two key results: (1) that \qpotsdoe~is asymptotically consistent and (2) under a finite sample $\mcl{D}_n$, any subset of $M$ oracle evaluations chosen per \Cref{eqn:inner_subset_mi} maximizes the one-step expected entropy reduction about the omitted oracles. W.l.o.g, we consider $K$ objectives without constraints. The informal assumptions and results are presented below; but detailed versions, with proofs, are in \Cref{sec:proofs}.

\paragraph{Assumptions.}
We assume that the design space \(\mathcal X\) is compact and that the \(K\) latent oracles are generated by the MTGP model used by the algorithm, with
continuous sample paths. Observations may be decoupled: at round \(i\), the algorithm queries a
design \(\mathbf{x}_i\) but observes only a subset \(S_i\subseteq [K]\), with independent Gaussian
observation noise of finite, nonzero variance for each observed oracle. The only exploration
condition is a no-starvation requirement: every oracle must be directly evaluated on design points
that become dense in \(\mathcal X\). Equivalently, decoupling is allowed, but no task may be
permanently ignored -- this condition can be enforced by a lightweight safeguard such as occasional
coupled evaluations (see line $16$ in \Cref{alg:qpots-doe}).

\begin{theorem}[Informal: consistency of \qpotsdoe]
\label{thm:informal-consistency}
Under Assumption~\ref{ass:regularity-no-starvation}, decoupled evaluations do
not prevent posterior consistency. In particular, if every task is directly
evaluated on a set of design points that becomes dense in $\mathcal X$, then
the multitask GP posterior learned from the partially observed data
$\mathcal D_n$ concentrates on the true task values. Equivalently, for any
fixed design point $\mathbf x\in\mathcal X$ and every task $k\in[K]$, the
posterior mean converges to the true value $Y_{0,k}(\mathbf x)$ and the
posterior uncertainty vanishes as $n\to\infty$.
\end{theorem}

\begin{theorem}[Informal: finite-sample value of decoupling]
\label{thm:informal_finite_sample_decoupling}
Fix the current data \(\mcl D_n\) and a candidate design \(\mathbf{x}\). If only \(M<K\) oracles
can be evaluated at \(\mathbf{x}\), then the \qpotsdoe~subset rule
\[
S_n^\star(\mathbf{x})
\in
\arg\max_{S\subseteq [K]:\, |S|=M}
I\!\left(Y_S(\mathbf{x});Y_{\bar S}(\mathbf{x})\mid \mcl{D}_n\right)
\]
selects the \(M\) oracles whose observation gives the largest expected one-step reduction
in posterior uncertainty about the unevaluated oracles \(Y_{\bar S}(\mathbf{x})\).
\end{theorem}

\section{Empirical results}
\label{sec:experiments}
We demonstrate our methodology on several synthetic and real-world benchmarks varying in input ($d$), output ($K$), and constraint $(C)$ dimensionality. We focus on evaluating high $K+C$ dimensionality since the benefit of decoupling is the primary objective of this paper; high $d$ scaling is focus of future work.
We compare \qpotsdoe~ against qPOTS~\cite{renganathan2025qpots} (no decoupling) and qLogNEHVI~\cite{daulton2020differentiable, ament2023unexpected} (a classical baseline, also with no decoupling) and four decoupled baselines: HVKG~\cite{daulton_hypervolume_nodate} (our main competitor), PESMO~\cite{garrido2016pesmoc}, MESMO~\cite{belakaria2019max}, and JESMO~\cite{tu2022joint}, all under uniform costs.  

For {\bf synthetic experiments}, we consider the Branin-Currin, ZDT1/2, DTLZ4/5, Gaussian mixture model (GMM), BNH~\cite{garrido2020parallel}, OSY~\cite{osyczka1995new}, and CONSTR~\cite{garrido2020parallel} test functions. Additional test functions are included in the supplementary material.  For {\bf real-world experiments}, we consider the Vehicle Safety problem~\cite{tanabe2020easy}, the Penicillin production problem~\cite{liang2021scalable}, the Car side-impact problem~\cite{jain2013evolutionary}, and the disc brake problem~\cite{tanabe2020easy}. Further details on all experiments, including the reference point $\mbf{r}$ (\Cref{tab:test-cases-summary}), are provided in the Appendix. 

We use MTGPs with an
ICM kernel~\cite{bonilla2008multi} for task correlation and the anisotropic Matern for input correlation, for qPOTS and \qpotsdoe, and independent ``single task'' GPs for all other methods (as they are designed to be). 
We provide each experiment with a set of $n=10\times d$ seed samples chosen uniformly at random from $\mcl{X}$, and \emph{fully coupled} observations to start the algorithm; then they are repeated $10$ times to observe mean and $\max/\min$ performance. Our metric for comparison is the hypervolume indicator~\cite{yang2019efficient} and the number of oracle evaluations.
 Our implementation primarily builds on GPyTorch~\cite{gardner2018gpytorch} and BoTorch~\cite{balandat2019botorch}; we leverage 
MPI4Py~\cite{dalcin2021mpi4py} to parallelize our repetitions on CPUs. The full source code is available at \url{https://github.com/csdlpsu/qpots}.

\begin{figure}[tb]
\centering
    \includegraphics[width=.55\textwidth]{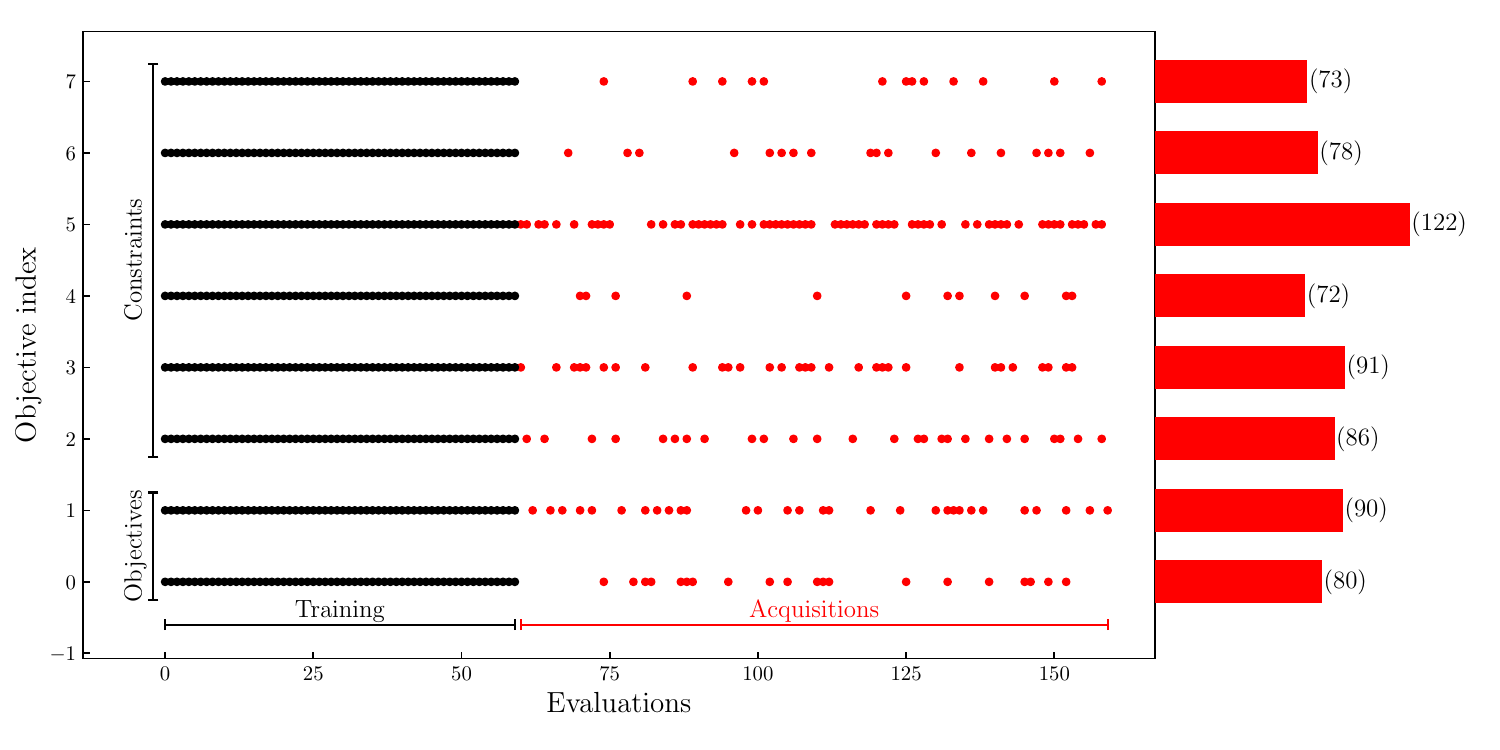}
    \caption{\qpotsdoe~decoupled evaluations (OSY)}
    \label{fig:osy_oracle_evals}
\end{figure}
The results of the synthetic experiments are shown in \Cref{fig:synthetic} -- notice that \qpotsdoe~is the clear winner in $6/8$ experiments (BNH, DTLZ4-5, ZDT2, OSY, CONSTR); in the remaining $2$, it is still no worse than others; crucially, \qpotsdoe~outperforms HVKG in all of them. Note that PESMO/JESMO/MESMO were run until we ran out of memory, and hence are missing in some experiments. 
The real-world experiments are shown in \Cref{fig:realworld}, where \qpotsdoe~beats all competitors in all cases. 
\Cref{fig:osy_oracle_evals} shows how \qpotsdoe~picks the objective and constraint evaluation at each round, for the OSY problem ($K=2, C=6$) at $q=1$ (additional illustrations are shown in \Cref{fig:osy_evaluations}). In this example, if decoupling were not allowed, then the total objective and constraint evaluations would have been $160\times 8 = 1280$ -- on the other hand, \qpotsdoe~required only $692$ evaluations total ($46\%$ fewer evaluations).

{\bf Ablations.} We set $q=4$ for all experiments; however, an ablation on $q$ (see \Cref{fig:q_ablations}) shows that \qpotsdoe~maintains its performance across $q$~\footnote{Note that the perceived difference in the plots is an artifact of the changing x-axis scale.}. We also did an ablation on $\tau$, varying it in $[10^{-6}, 1]$ for the OSY and Branin-Currin experiments (\Cref{fig:ablations}); note that when $\tau=\infty$, \qpotsdoe~reduces to \qpots~as expected. We found that any choice between $[10^{-6}, 1)$ had a marginal impact on performance; as a default choice, we set (and recommend) $\tau = 10^{-4}$ in our method. We observe that identifying an optimal $\tau$ for each problem is not critical for the overall performance. We include additional ablations on the kernel choice, run-time comparisons, and the benefit of multitask over single-task GPs, in the Appendix.

\begin{figure*}[htbp]
    \centering
    \begin{subfigure}{.25\textwidth}
        \centering
        \includegraphics[width=1\linewidth]{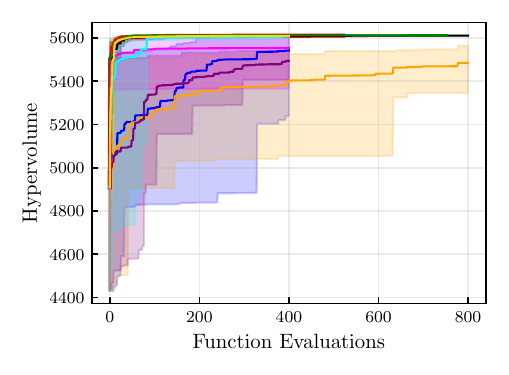}
        \caption{Branin-Currin $2/2/-$}
    \end{subfigure}%
    \begin{subfigure}{.25\textwidth}
        \centering
        \includegraphics[width=1\linewidth]{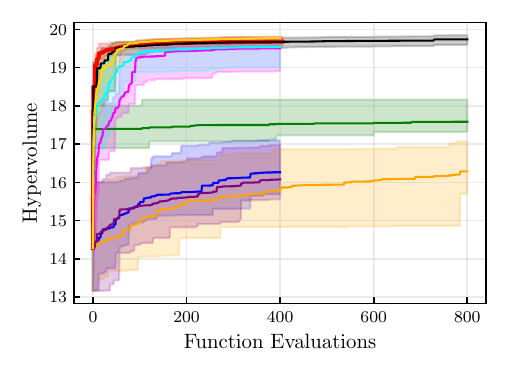}
        \caption{ZDT1 $2/5/-$}
    \end{subfigure}%
    \begin{subfigure}{.25\textwidth}
        \centering
        \includegraphics[width=1\linewidth]{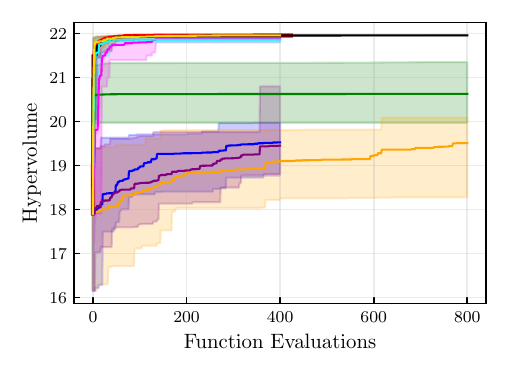}
        \caption{ZDT2 $2/5/-$}
    \end{subfigure}%
    \begin{subfigure}{.25\textwidth}
        \centering
        \includegraphics[width=1\linewidth]{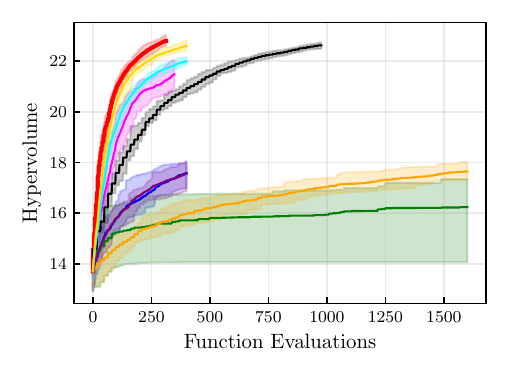}
        \caption{DTLZ4 $4/5/-$}
    \end{subfigure}\\
    \begin{subfigure}{.25\textwidth}
        \centering
        \includegraphics[width=1\linewidth]{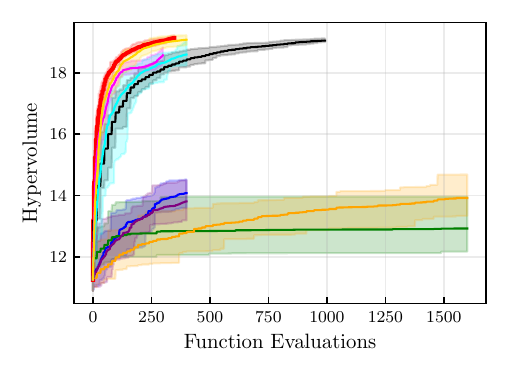}
        \caption{DTLZ5 $4/5/-$}
    \end{subfigure}%
    \begin{subfigure}{.25\textwidth}
        \centering
        \includegraphics[width=1\linewidth]{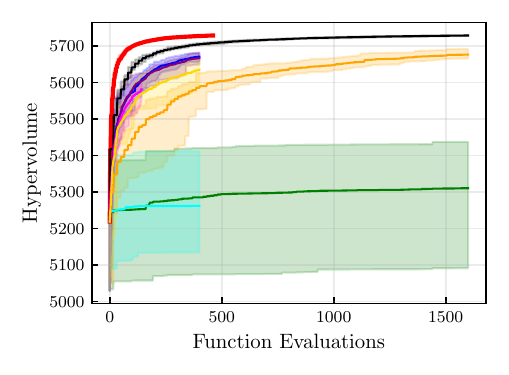}
        \caption{BNH $2/2/2$}
    \end{subfigure}%
    \begin{subfigure}{.25\textwidth}
        \centering
        \includegraphics[width=1\linewidth]{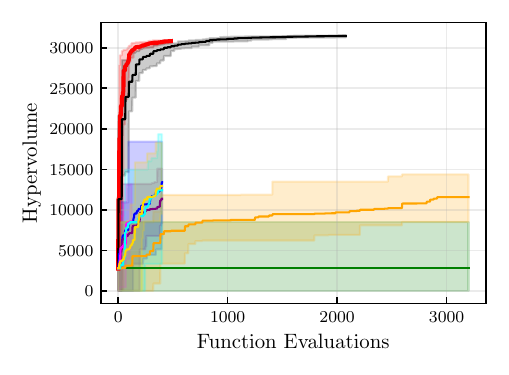}
        \caption{OSY $6/2/6$}
    \end{subfigure}%
    \begin{subfigure}{.25\textwidth}
        \centering
        \includegraphics[width=1\linewidth]{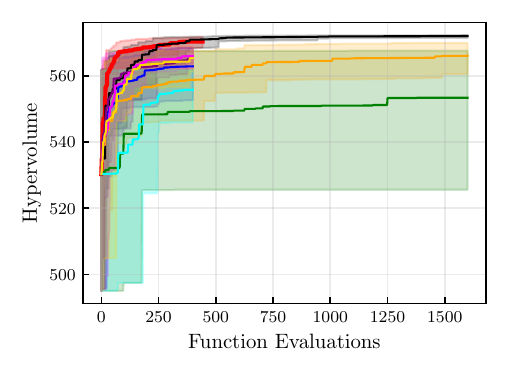}
        \caption{CONSTR $2/2/2$}
    \end{subfigure}\\
    \begin{subfigure}{1\textwidth}
        \includegraphics[width=1\linewidth]{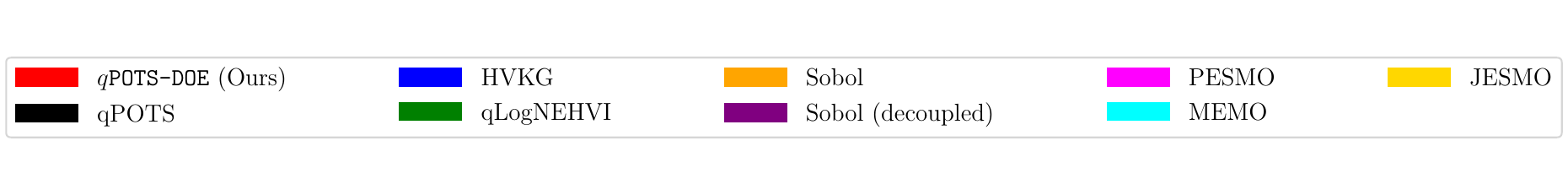}
    \end{subfigure}          
    \caption{Synthetic experiments. The numbers in the labels follow $d/K/C$ dimensions. Solid lines are the mean and filled colors are the $\min/\max$ range over $10$ repetitions.}
    \label{fig:synthetic}
\end{figure*}

\begin{figure*}[htbp]
    \centering
    \begin{subfigure}{.25\textwidth}
        \centering
        \includegraphics[width=1\linewidth]{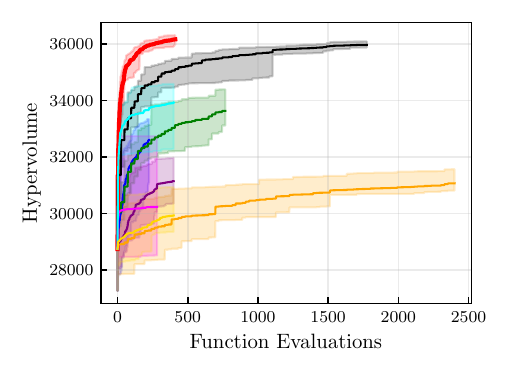}
        \caption{Car side-impact $7/4/-$}
    \end{subfigure}%
    \begin{subfigure}{.25\textwidth}
        \centering
        \includegraphics[width=1\linewidth]{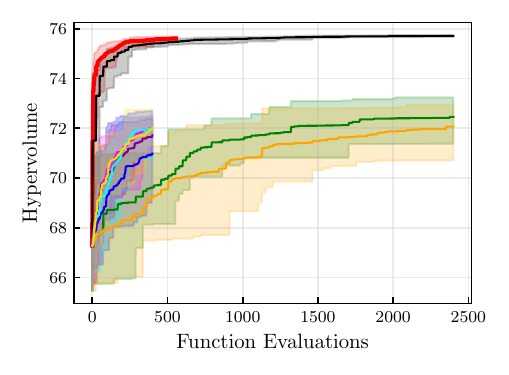}
        \caption{Disc brake $4/2/4$}
    \end{subfigure}%
    \begin{subfigure}{.25\textwidth}
        \centering
        \includegraphics[width=1\linewidth]{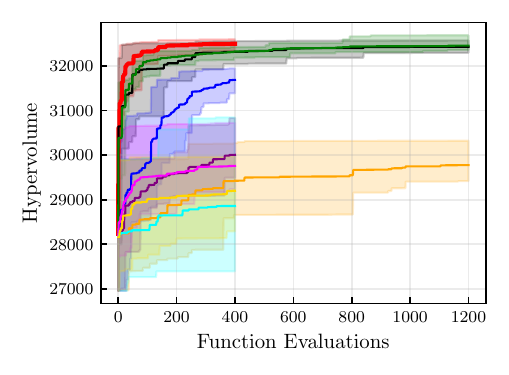}
        \caption{Vehicle safety $5/3/-$}
    \end{subfigure}%
    \begin{subfigure}{.25\textwidth}
        \centering
        \includegraphics[width=1\linewidth]{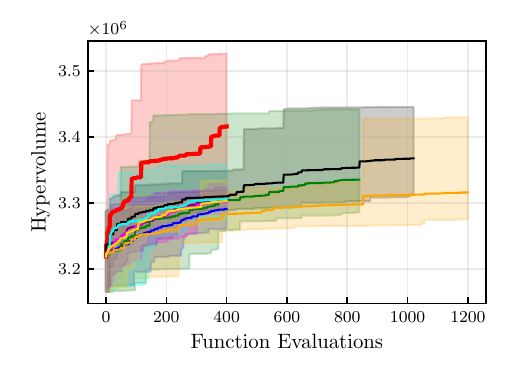}
        \caption{Penicillin $7/3/-$}
    \end{subfigure}\\
    \begin{subfigure}{1\textwidth}
        \includegraphics[width=1\linewidth]{figures/legend_two_row.pdf}
    \end{subfigure}    
    \caption{Real-world experiments. Solid lines are the mean and filled colors are the $\min/\max$ range over $10$ repetitions.}
    \label{fig:realworld}
\end{figure*}

\begin{figure*}
    \centering
        \begin{subfigure}{.33\textwidth}
        \centering
        \includegraphics[width=1\linewidth]{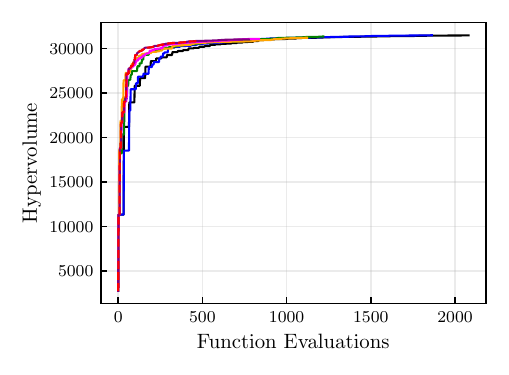}
        \caption*{OSY}
    \end{subfigure}%
    \begin{subfigure}{.33\textwidth}
        \centering
        \includegraphics[width=1\linewidth]{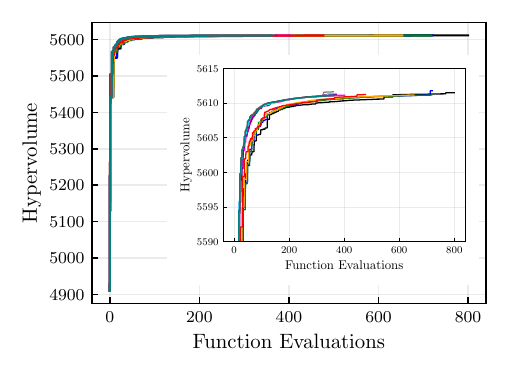}
        \caption*{Branin-Currin (inset: zoom-in)}
    \end{subfigure}%
    \begin{subfigure}{.34\textwidth}
        \centering
        \includegraphics[width=1\linewidth]{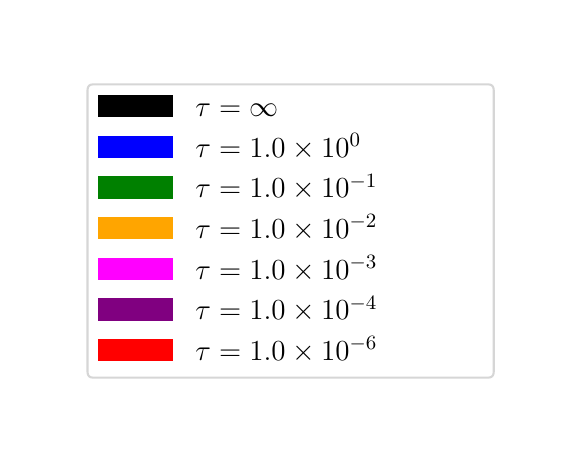}
        \caption*{}
    \end{subfigure}%
    \caption{{\bf Ablation studies.} Top row: ablation on $q$ (OSY problem); bottom row: ablation on $\tau$. 
    }
    \label{fig:ablations}
\end{figure*}

\section{Conclusions}
\label{sec:conclusions}
We introduced \qpotsdoe, a batch Thompson sampling framework for constrained multiobjective Bayesian optimization that exploits correlations among objectives and constraints to support decoupled oracle evaluations under uniform evaluation costs. By replacing independent surrogate models with a multitask GP, \qpotsdoe~preserves cross-task posterior structure; by using total correlation as a decoupling gate and Gaussian mutual information to select which objectives or constraints to evaluate, it turns correlation into a practical mechanism for reducing redundant evaluations. Theoretical analysis shows that the resulting policy remains asymptotically consistent under dense-design and no-starvation , and that our subset rule for decoupling is optimal, while the empirical results on synthetic and real-world benchmarks indicate that \qpotsdoe~improves hypervolume performance relative to several existing coupled and decoupled baselines. Overall, the method broadens the role of decoupled evaluations in MOBO: decoupling need not rely only on unequal costs, but can also arise from statistical dependence among expensive oracles. Future work should focus on high (input) dimensional scaling, testing on additional gains due to variable costs, and robustness to misspecified task correlations.

{\bf Limitations}. \qpotsdoe~is most effective when objectives and constraints exhibit exploitable posterior correlations; when such correlations are weak or poorly captured by the MTGP kernel, the gains from decoupling may diminish, although the TC threshold allows the method to revert toward coupled evaluation. The method also introduces hyperparameters, notably the TC threshold and the decoupled task budget, and while our ablations suggest limited sensitivity to the threshold, adaptive selection remains a direction to explore. 

\clearpage

\bibliographystyle{plainnat}
\bibliography{references}

\clearpage
\appendix

\section{Supplementary material}

\subsection{Additional experiments and ablations}

\paragraph{Details on experiments.}
The Penicillin production problem ($d=7, K=3$) seeks to maximize Penicillin yield while minimizing both the production time and the carbon dioxide byproduct emission. The $d=7$ input variables include the culture medium volume, biomass concentration, temperature (in K), glucose substrate concentration, substrate feed rate, substrate feed concentration and the $H+$ concentration~\cite{liang2021scalable}. The Vehicle Safety problem ($d=5, K=3$)~\cite{tanabe2020easy} concerns the design of automobiles for enhanced safety. Specifically, we consider the minimization of vehicle mass, the vehicle acceleration during a full-frontal crash, and the toe-board intrusion---a measure of the mechanical damage to the vehicle during an off-frontal crash, with respect to $d=5$ design variables which represent the reinforced parts of the frontal frame; further details are found in~\cite{liao2008multiobjective}. The Car side-impact problem is aimed at minimizing the weight of the car while minimizing the pubic force experienced by a passenger and the
average velocity of the V-pillar responsible for withstanding the impact load. Additionally, as the fourth objective, the maximum constraint violation of $10$ constraints which include limiting values of abdomen load, pubic
force, velocity of V-pillar, and rib deflection is also added~\cite{jain2013evolutionary}.

\paragraph{Experimental settings.}
Additionally, we also include Sobol sampling~\cite{sobol1967distribution} with and without decoupling. We implemented decoupled variants of PESMO, MESMO, and JESMO by introducing an evaluation mask that specifies which objective is observed at a candidate design, and by optimizing over the input ($\x$) - objective index ($k$) pairs rather than over $\x$ alone. For PESMO, the entropy reduction is restricted to the masked objective term, while for MESMO and JESMO we use a projection-based approximation that conditions only on the observed objective by projecting the posterior and sampled Pareto geometry onto the corresponding subspace and recomputing the lower-dimensional box decomposition. 

\begin{algorithm}[t]
\caption{\qpotsdoe: Batch Pareto optimal Thompson sampling with Decoupled Oracle Evaluations}
\label{alg:qpots-doe}
\begin{algorithmic}[1]
\Require Initial data $\mathcal{D}_0$; objective and constraint oracles $f_k,~k=1,2,3,\ldots$ and $c_j,~j\in \mcl{I} \cup \mcl{E}$;
batch size $q$; decoupling threshold $\tau$; rounds $T$
\For{$i = 0,\ldots,T-1$}
    \State Fit one multitask Gaussian process to all observations in $\mathcal{D}_i$
    \State Draw one joint Thompson sample of all objectives and constraints
    \State Find the feasible Pareto set under this sampled model
    \State Select $q$ diverse candidate designs from the sampled Pareto set
    \For{each candidate design $\x$}
        \State Compute a dependence score $\mathrm{TC}(\x)$ between the objective/constraint tasks
        \If{$\mathrm{TC} < \tau$}
            \State Evaluate all objectives and constraints at $\x$
            \Comment{coupled evaluation}
        \Else
            \State Identify $M$, and choose $M$ tasks that are most informative about the remaining tasks
            \State Evaluate only those selected tasks at $\x$
            \Comment{decoupled evaluation}
        \EndIf
        \State Add the new observations to $\mathcal{D}_{i+1}$
    \EndFor
    \State (optional) Occasionally force rarely observed tasks to be evaluated
    \Comment{prevents task starvation}
\EndFor
\State \Return The feasible designs and estimated Pareto frontier from the final model
\end{algorithmic}
\end{algorithm}
\begin{figure*}[p]
    \centering
    \begin{subfigure}{.28\textwidth}
        \centering
        \includegraphics[width=1\linewidth]{figures/branincurrin_q4_HV_History_Scaled.pdf}
        \caption{BC $2/2/-$}
    \end{subfigure}%
    \begin{subfigure}{.28\textwidth}
        \centering
        \includegraphics[width=1\linewidth]{figures/bnh_HV_History_Scaled.pdf}
        \caption{BNH $2/2/2$}
    \end{subfigure}%
    \begin{subfigure}{.28\textwidth}
        \centering
        \includegraphics[width=1\linewidth]{figures/Negate=False_dtlz4_dim5_obj4_HV_History_Scaled.pdf}
        \caption{DTLZ4 $5/4/-$}
    \end{subfigure}\\
    \begin{subfigure}{.28\textwidth}
        \centering
        \includegraphics[width=1\linewidth]{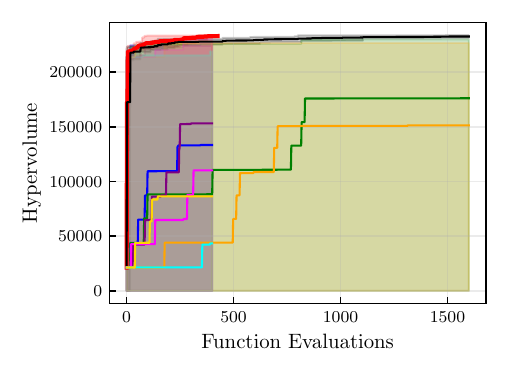}
        \caption{SRN $2/2/2$}
    \end{subfigure}%
    \begin{subfigure}{.28\textwidth}
        \centering
        \includegraphics[width=1\linewidth]{figures/Negate=False_dtlz5_dim5_obj4_HV_History_Scaled.pdf}
        \caption{DTLZ5 $5/4/-$}
    \end{subfigure}%
    \begin{subfigure}{.28\textwidth}
        \centering
        \includegraphics[width=1\linewidth]{figures/constr_HV_History_Scaled.pdf}
        \caption{CONSTR $2/2/2$}
    \end{subfigure}\\
    \begin{subfigure}{.28\textwidth}
        \centering
        \includegraphics[width=1\linewidth]{figures/Negate=False_zdt1_dim5_obj2_HV_History_Scaled.pdf}
        \caption{ZDT1 $5/2/-$}
    \end{subfigure}%
    \begin{subfigure}{.28\textwidth}
        \centering
        \includegraphics[width=1\linewidth]{figures/Negate=False_zdt2_dim5_obj2_HV_History_Scaled.pdf}
        \caption{ZDT2 $5/2/-$}
    \end{subfigure}%
    \begin{subfigure}{.28\textwidth}
        \centering
        \includegraphics[width=1\linewidth]{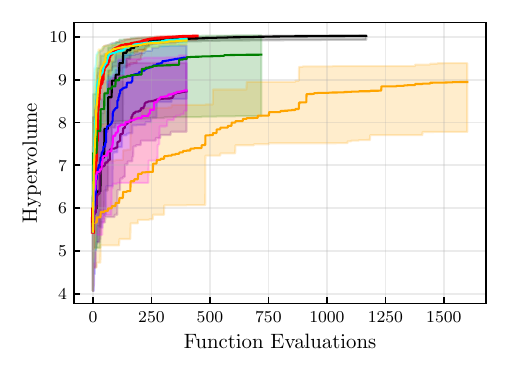}
        \caption{GMM $2/4/-$}
    \end{subfigure}\\
    \begin{subfigure}{.28\textwidth}
        \centering
        \includegraphics[width=1\linewidth]{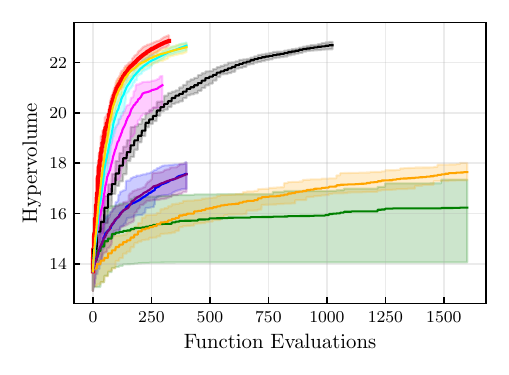}
        \caption{DTLZ2 $5/4/-$}
    \end{subfigure}%
    \begin{subfigure}{.28\textwidth}
        \centering
        \includegraphics[width=1\linewidth]{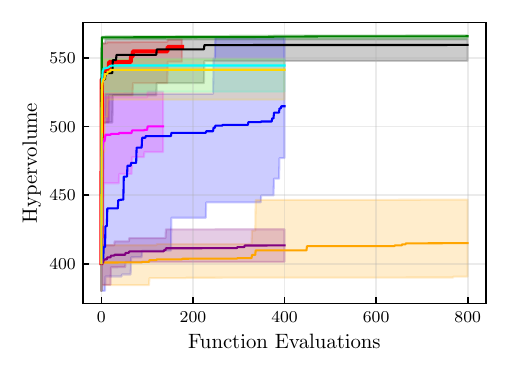}
        \caption{DTLZ7 $20/2/-$}
    \end{subfigure}%
    \begin{subfigure}{.28\textwidth}
        \centering
        \includegraphics[width=1\linewidth]{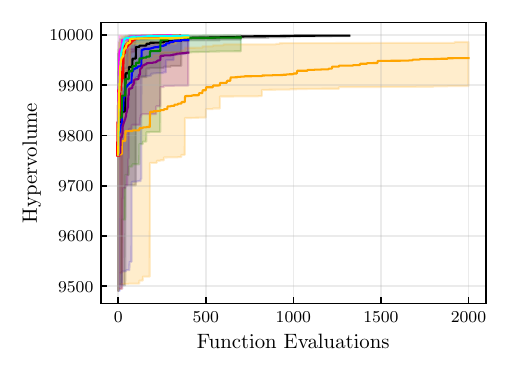}
        \caption{C2DTLZ2 $5/4/1$}
    \end{subfigure}\\
    \begin{subfigure}{.28\textwidth}
        \centering
        \includegraphics[width=1\linewidth]{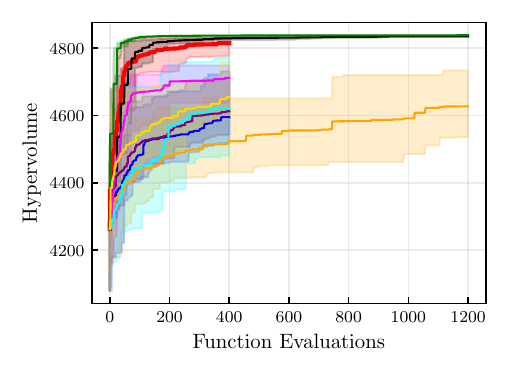}
        \caption{Cons. Branin-Currin $2/2/1$}
    \end{subfigure}%
    \begin{subfigure}{.28\textwidth}
        \centering
        \includegraphics[width=1\linewidth]{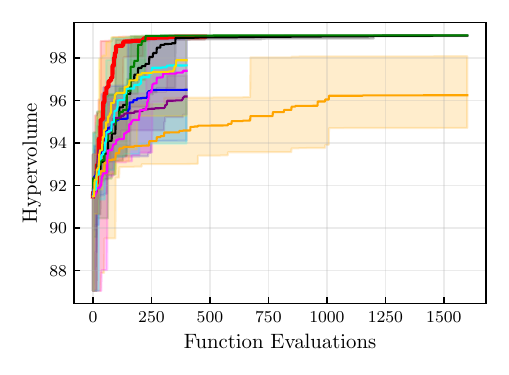}
        \caption{MW7}
    \end{subfigure}%
    \begin{subfigure}{.28\textwidth}
        \centering
        \includegraphics[width=1\linewidth]{figures/osy_q4_HV_History_Scaled.pdf}
        \caption{OSY}
    \end{subfigure}\\
    \begin{subfigure}{.85\textwidth}
        \includegraphics[width=1\linewidth]{figures/legend_two_row.pdf}
    \end{subfigure}    
    \caption{Synthetic experiments}
    \label{fig:synthetic_app}
\end{figure*}

\begin{figure*}[htbp]
    \centering
    \begin{subfigure}{.5\textwidth}
        \centering
        \includegraphics[width=1\linewidth]{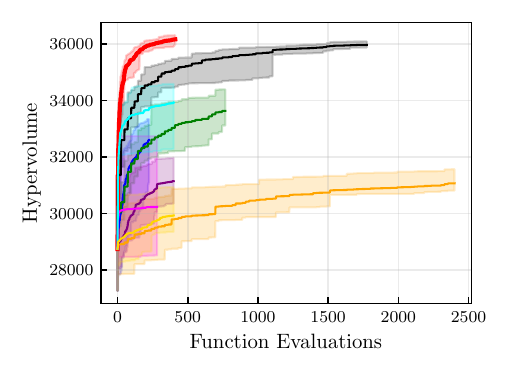}
        \caption{Car side-impact}
    \end{subfigure}%
    \begin{subfigure}{.5\textwidth}
        \centering
        \includegraphics[width=1\linewidth]{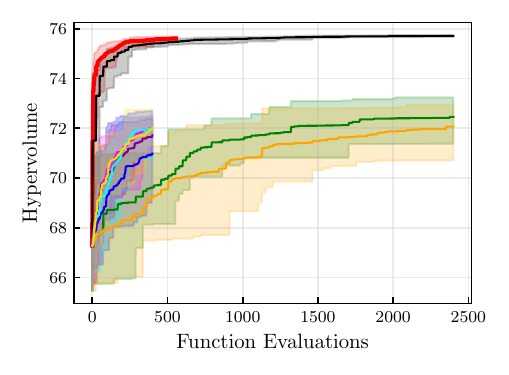}
        \caption{Disc brake}
    \end{subfigure}\\
    \begin{subfigure}{.5\textwidth}
        \centering
        \includegraphics[width=1\linewidth]{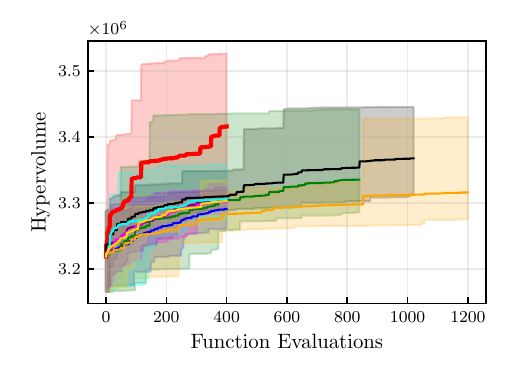}
        \caption{Penicillin}
    \end{subfigure}%
    \begin{subfigure}{.5\textwidth}
        \centering
        \includegraphics[width=1\linewidth]{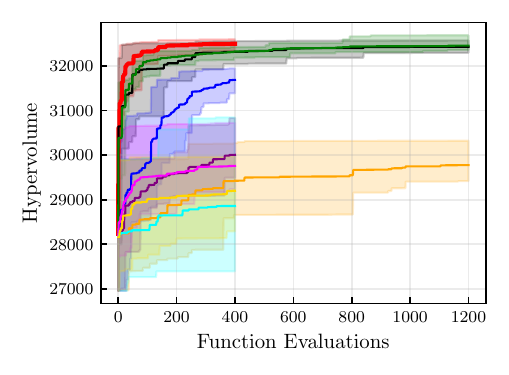}
        \caption{Vehicle safety}
    \end{subfigure}\\
    \begin{subfigure}{1\textwidth}
        \includegraphics[width=1\linewidth]{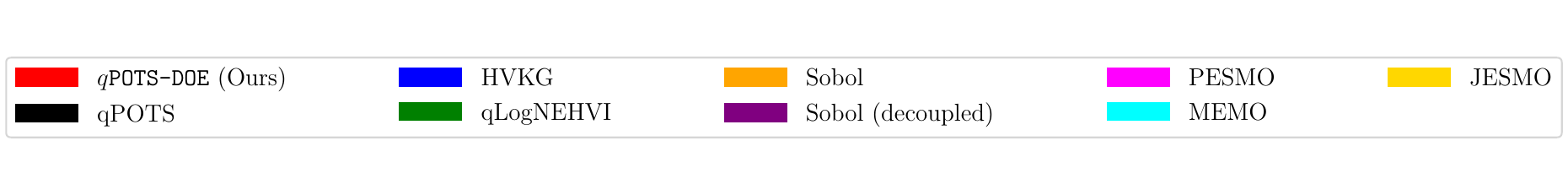}
    \end{subfigure}  
    \label{fig:realworld_app}
    \caption{Real-world experiments.}
\end{figure*}

\begin{figure*}[htbp]
    \centering
    \begin{subfigure}{.25\textwidth}
        \centering
        \includegraphics[width=1\linewidth]{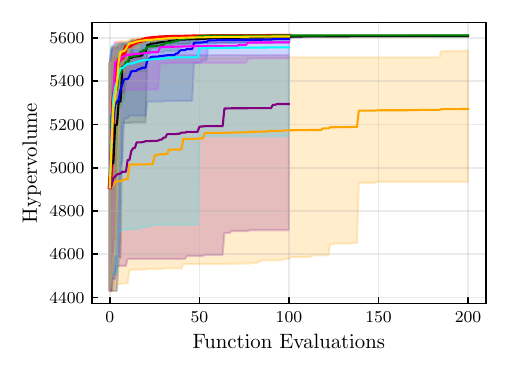}
        \caption*{}
    \end{subfigure}%
    \begin{subfigure}{.25\textwidth}
        \centering
        \includegraphics[width=1\linewidth]{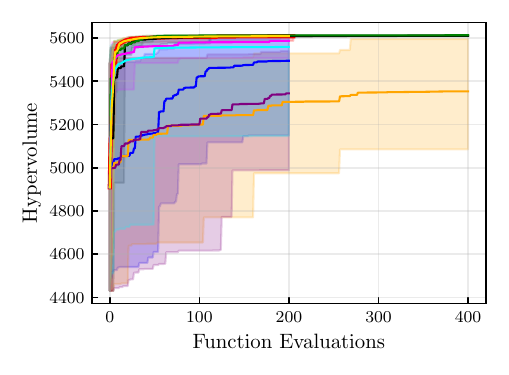}
        \caption*{}
    \end{subfigure}%
    \begin{subfigure}{.25\textwidth}
        \centering
        \includegraphics[width=1\linewidth]{figures/branincurrin_q4_HV_History_Scaled.pdf}
        \caption*{}
    \end{subfigure}%
    \begin{subfigure}{.25\textwidth}
        \centering
        \includegraphics[width=1\linewidth]{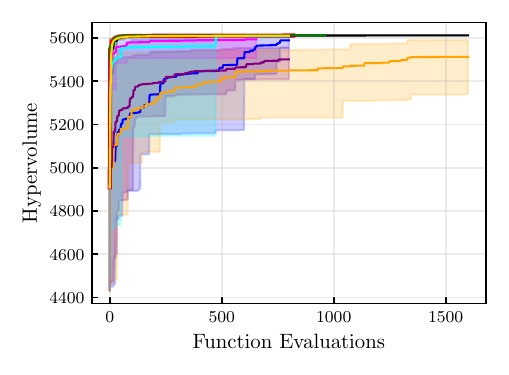}
        \caption*{}
    \end{subfigure}\\
    \begin{subfigure}{.25\textwidth}
        \centering
        \includegraphics[width=1\linewidth]{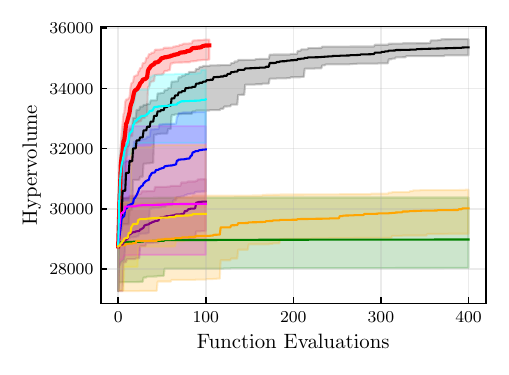}
        \caption*{}
    \end{subfigure}%
    \begin{subfigure}{.25\textwidth}
        \centering
        \includegraphics[width=1\linewidth]{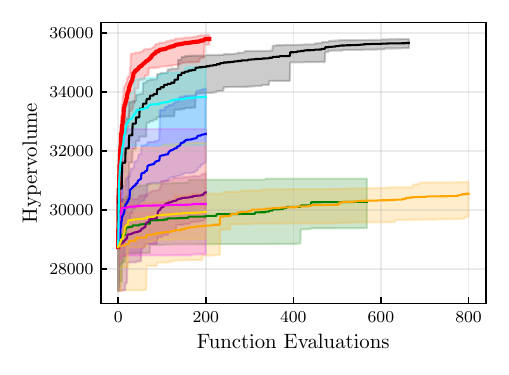}
        \caption*{}
    \end{subfigure}%
    \begin{subfigure}{.25\textwidth}
        \centering
        \includegraphics[width=1\linewidth]{figures/carside_q4_HV_History_Scaled.pdf}
        \caption*{}
    \end{subfigure}%
    \begin{subfigure}{.25\textwidth}
        \centering
        \includegraphics[width=1\linewidth]{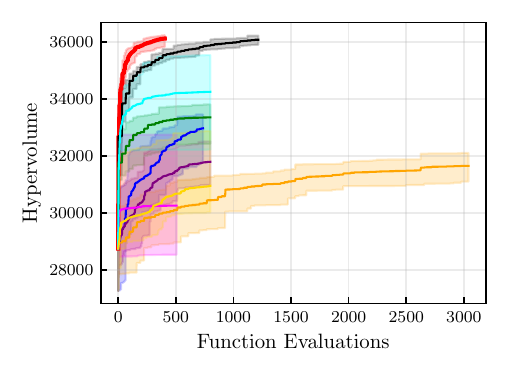}
        \caption*{}
    \end{subfigure}\\
    \begin{subfigure}{.25\textwidth}
        \centering
        \includegraphics[width=1\linewidth]{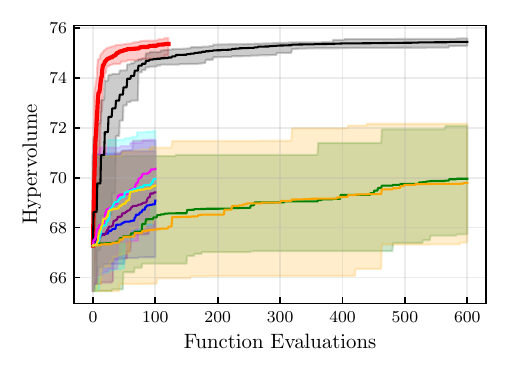}
        \caption*{}
    \end{subfigure}%
    \begin{subfigure}{.25\textwidth}
        \centering
        \includegraphics[width=1\linewidth]{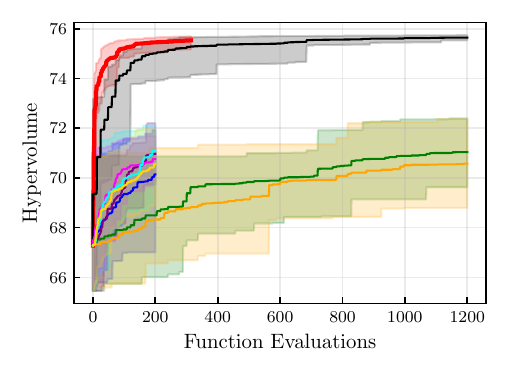}
        \caption*{}
    \end{subfigure}%
    \begin{subfigure}{.25\textwidth}
        \centering
        \includegraphics[width=1\linewidth]{figures/discbrake_q4_HV_History_Scaled.pdf}
        \caption*{}
    \end{subfigure}%
    \begin{subfigure}{.25\textwidth}
        \centering
        \includegraphics[width=1\linewidth]{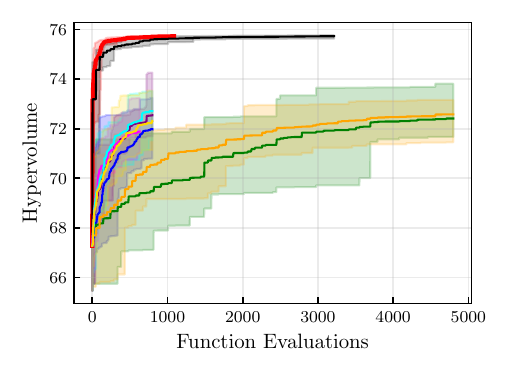}
        \caption*{}
    \end{subfigure}\\
    \begin{subfigure}{.25\textwidth}
        \centering
        \includegraphics[width=1\linewidth]{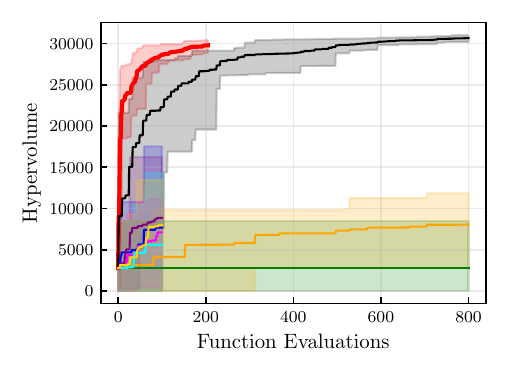}
        \caption{$q=1$}
    \end{subfigure}%
    \begin{subfigure}{.25\textwidth}
        \centering
        \includegraphics[width=1\linewidth]{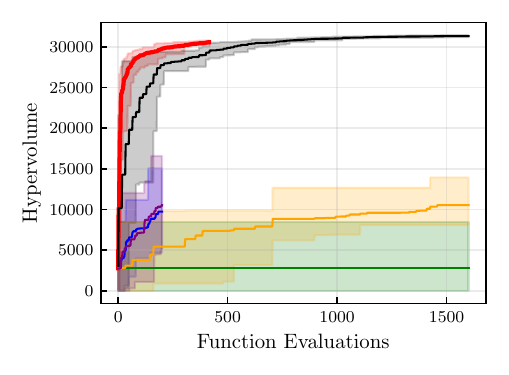}
        \caption{$q=2$}
    \end{subfigure}%
    \begin{subfigure}{.25\textwidth}
        \centering
        \includegraphics[width=1\linewidth]{figures/osy_q4_HV_History_Scaled.pdf}
        \caption{$q=4$}
    \end{subfigure}%
    \begin{subfigure}{.25\textwidth}
        \centering
        \includegraphics[width=1\linewidth]{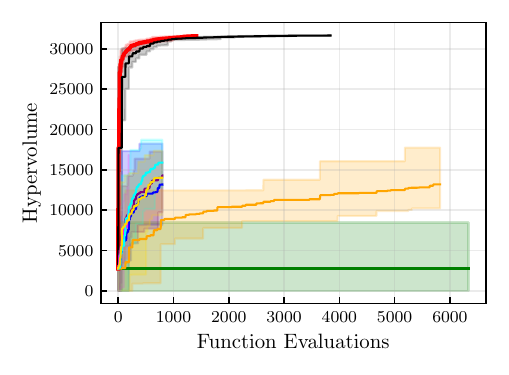}
        \caption{$q=8$}
    \end{subfigure}\\
    \begin{subfigure}{1\textwidth}
        \includegraphics[width=1\linewidth]{figures/legend_two_row.pdf}
    \end{subfigure}
    \label{fig:q_ablations}
    \caption{Ablation on batch size $q$. Row 1: Branin-Currin, row-2: Car side-impact, row-3: Disc brake, row-4: OSY.}
\end{figure*}

\begin{figure}[htbp]
    \centering
    \begin{subfigure}{.5\textwidth}
        \includegraphics[width=1\linewidth]{figures/osy_q1_scatter_evaluations_per_iteration.pdf}
        \caption{$q=1$}
    \end{subfigure}%
    \begin{subfigure}{.5\textwidth}
        \includegraphics[width=1\linewidth]{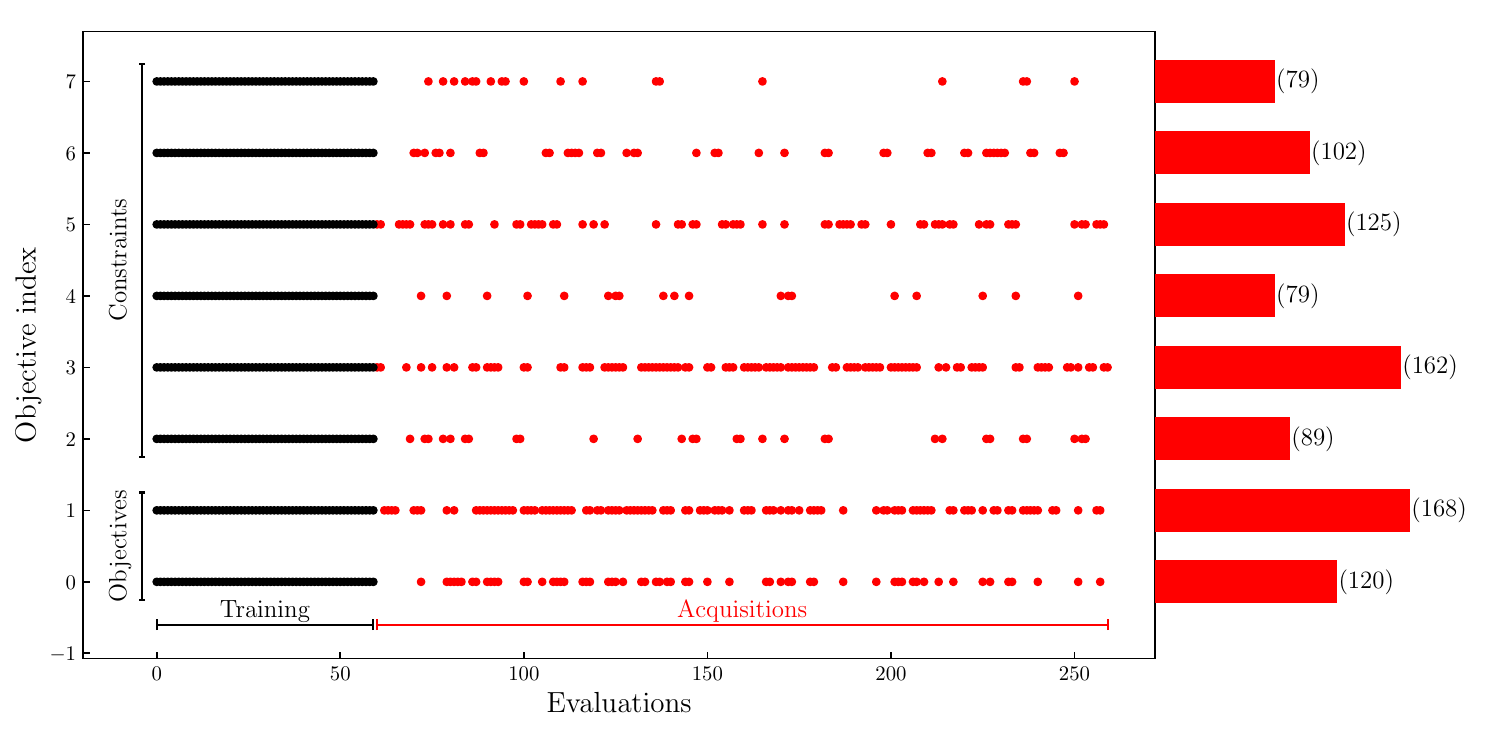}
        \caption{$q=2$}
    \end{subfigure}\\
    \begin{subfigure}{.5\textwidth}
        \includegraphics[width=1\linewidth]{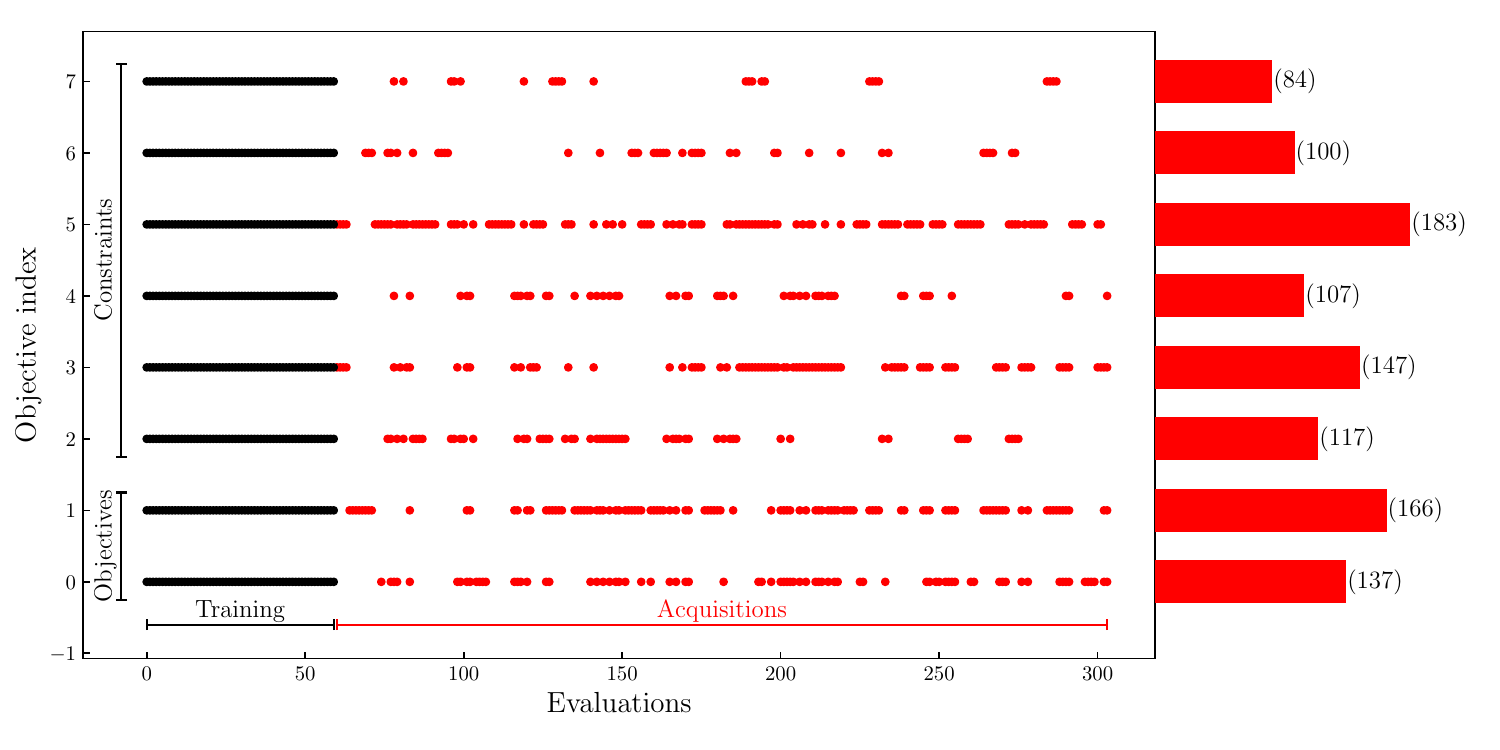}
        \caption{$q=4$}
    \end{subfigure}%
    \begin{subfigure}{.5\textwidth}
        \includegraphics[width=1\linewidth]{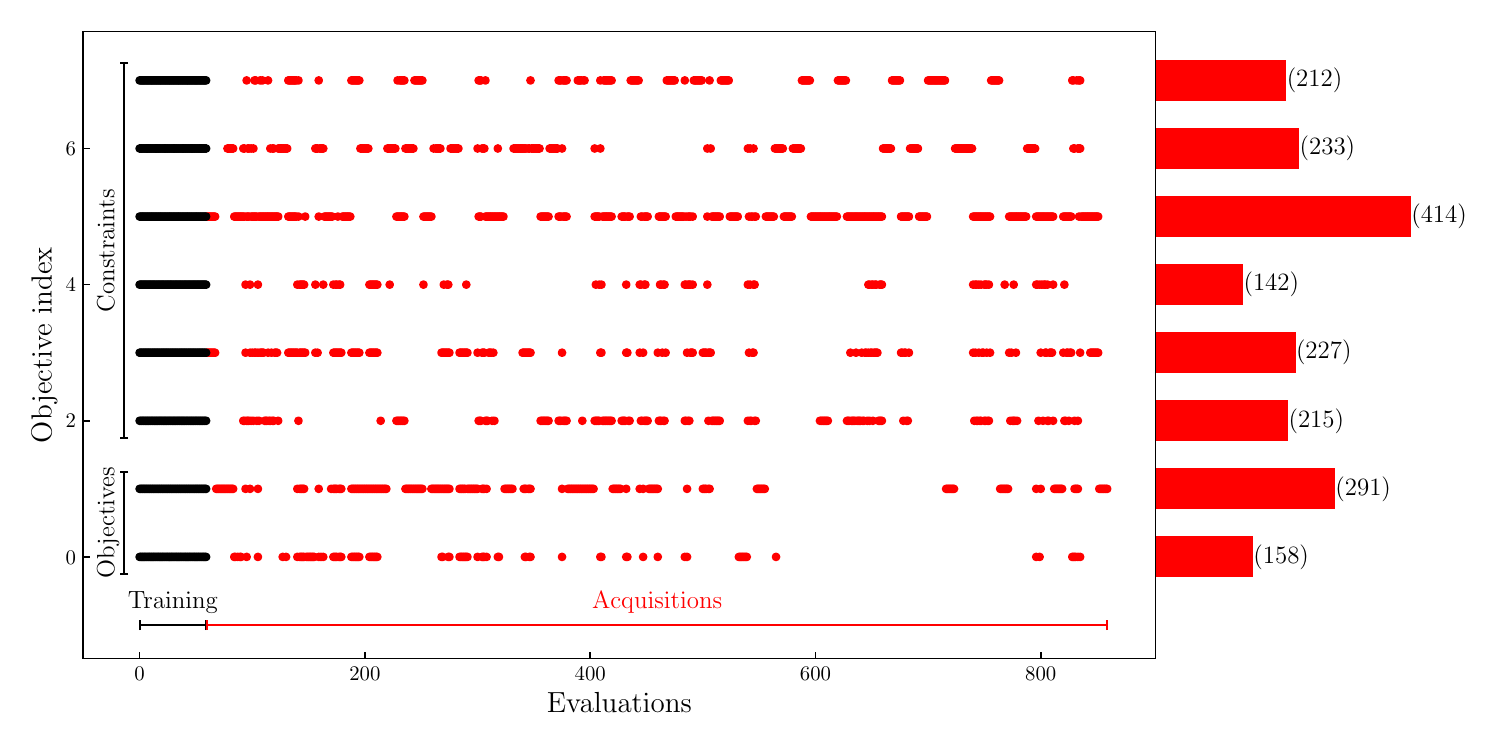}
        \caption{$q=8$}
    \end{subfigure}\\
    \caption{Distribution of decoupled oracle evaluations (OSY)}
    \label{fig:osy_evaluations}
\end{figure}

\begin{figure}[htbp]
    \centering
    \begin{subfigure}{.5\textwidth}
        \includegraphics[width=1\linewidth]{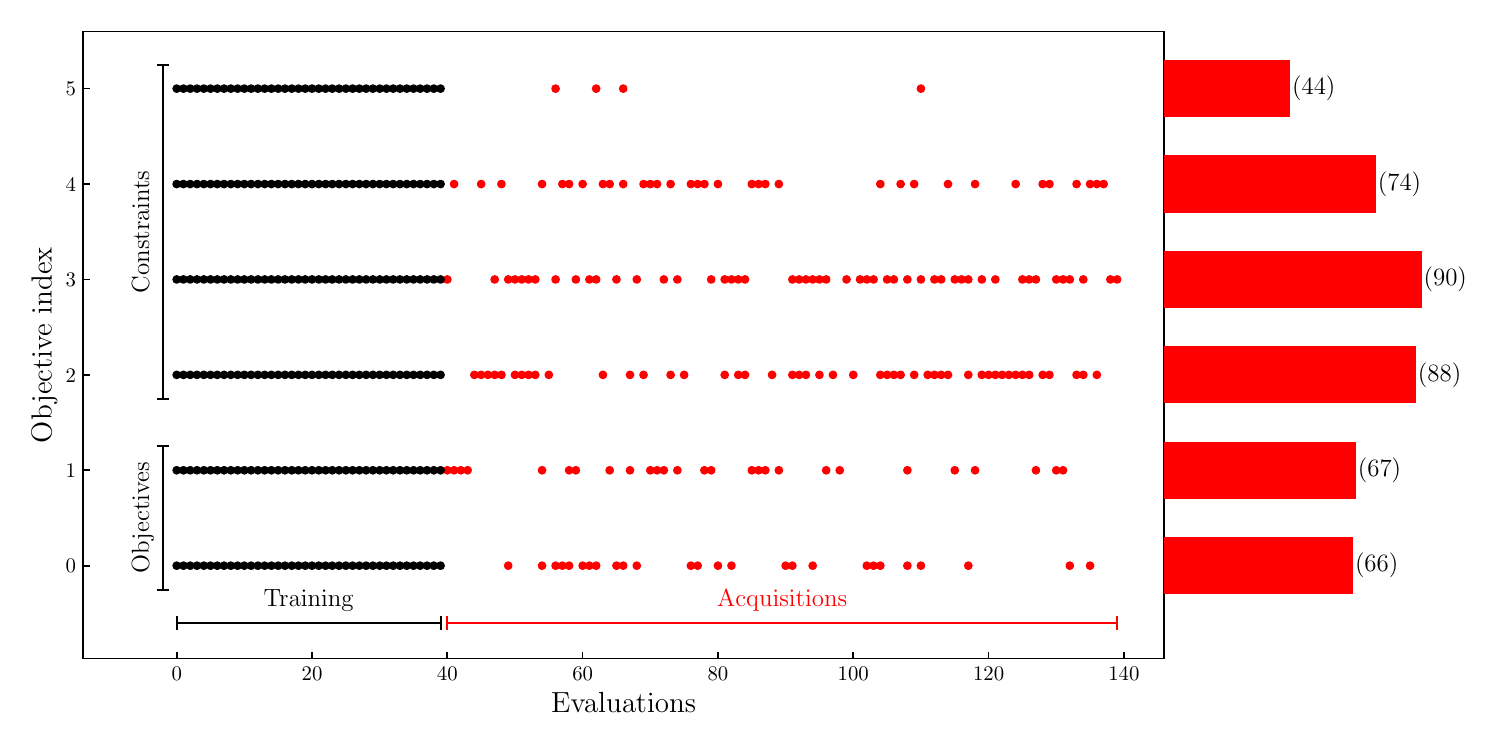}
        \caption{$q=1$}
    \end{subfigure}%
    \begin{subfigure}{.5\textwidth}
        \includegraphics[width=1\linewidth]{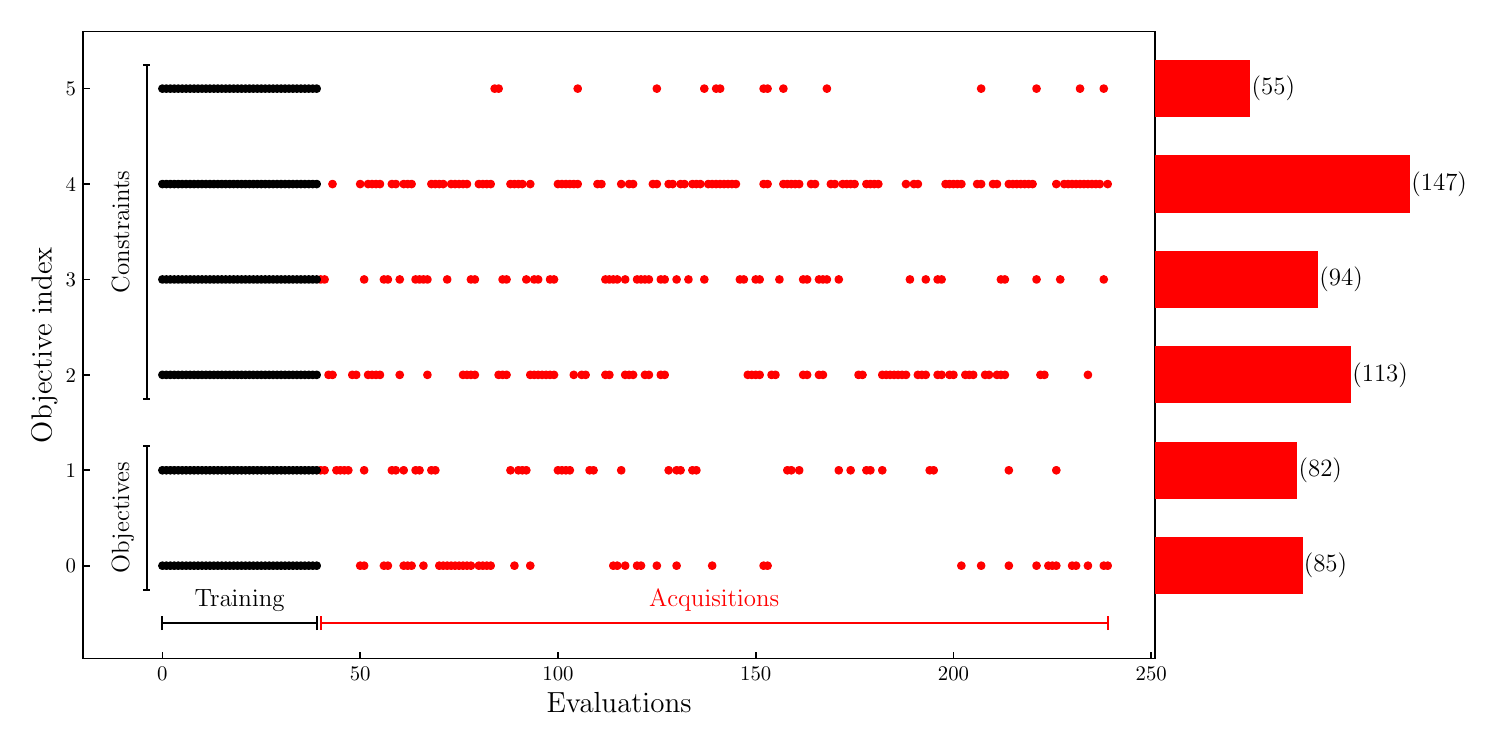}
        \caption{$q=2$}
    \end{subfigure}\\
    \begin{subfigure}{.5\textwidth}
        \includegraphics[width=1\linewidth]{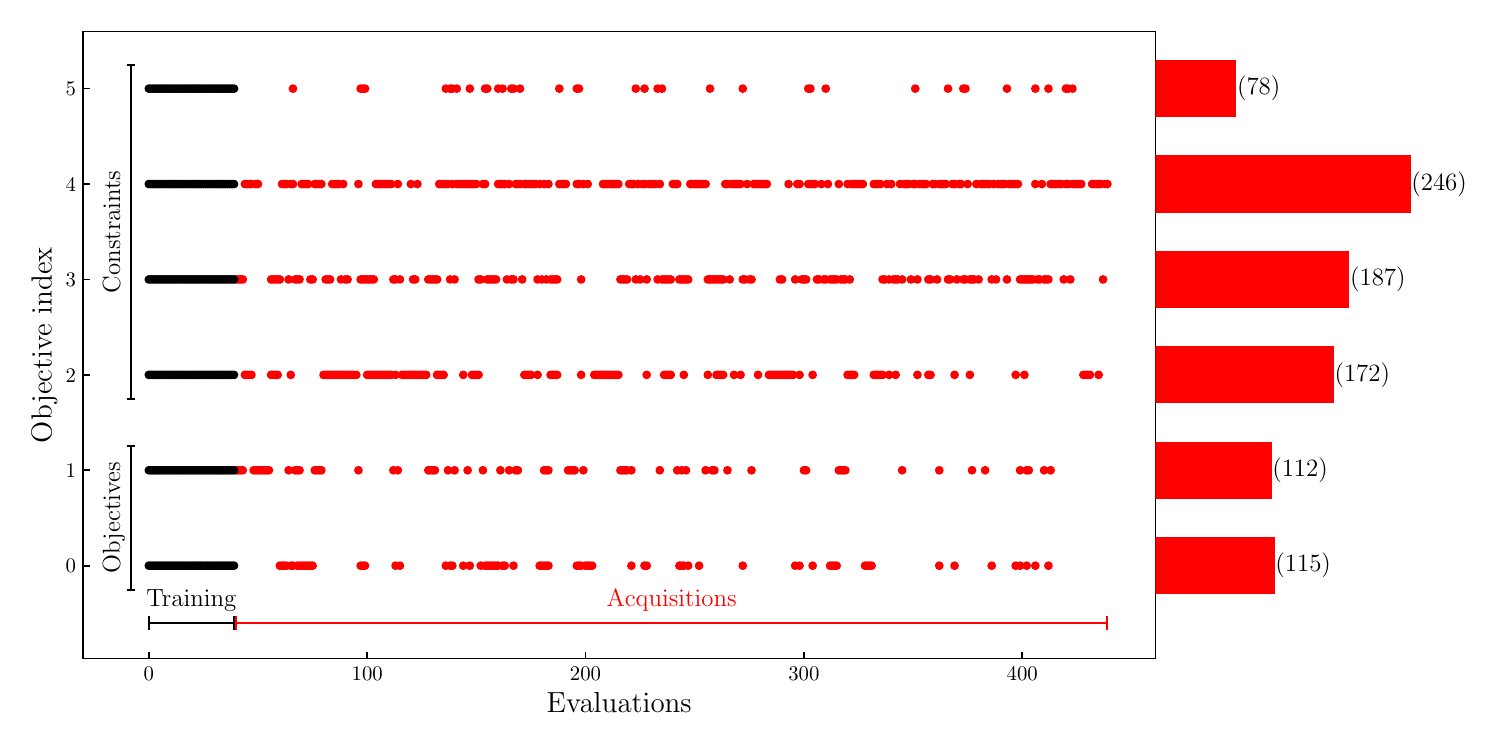}
        \caption{$q=4$}
    \end{subfigure}%
    \begin{subfigure}{.5\textwidth}
        \includegraphics[width=1\linewidth]{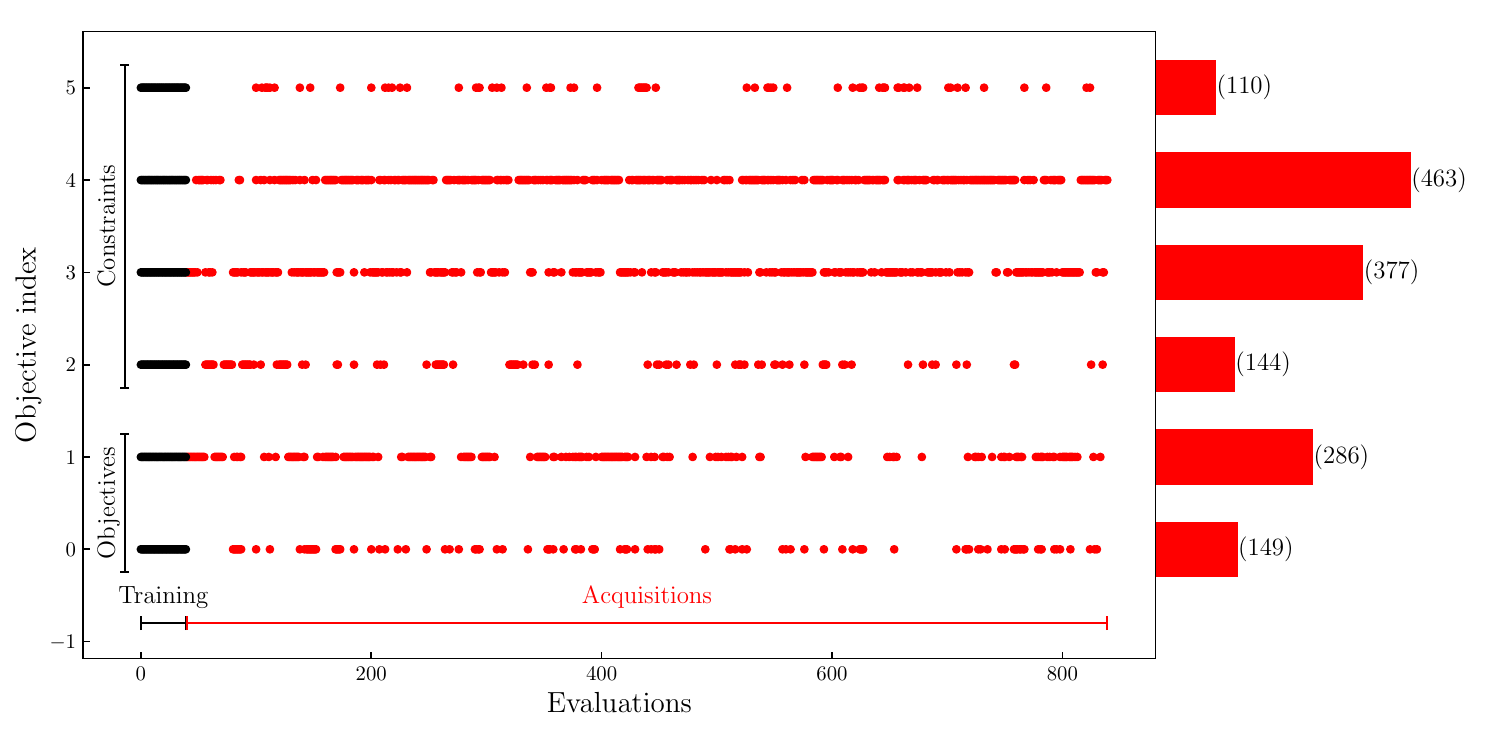}
        \caption{$q=8$}
    \end{subfigure}\\
    \caption{Distribution of decoupled oracle evaluations (Disc brake)}
    \label{fig:discbrake_evaluations}
\end{figure}

\begin{table}[htbp]
\centering
\small
\caption{Summary of all experiments}
\label{tab:test-cases-summary}
\begin{tabularx}{\textwidth}{@{}>{\raggedright\arraybackslash}X r r r >{\raggedright\arraybackslash}X l@{}}
\toprule
\textbf{Test Case} & \textbf{Dim} & \textbf{Obj} & \textbf{Cons} & \textbf{Ref Point} & \textbf{Ablations:} \\
\midrule
BNH & 2 & 2 & 2 & $(-150, -50)$ & -- \\
BraninCurrin & 2 & 2 & 0 & $(-300, -20)$ & $q=1,2,4,8$ \\
C2DTLZ2 & 5 & 4 & 1 & $(-10, -10, -10, -10)$ & -- \\
CarsideImpact & 7 & 4 & 0 & $(-50, -10, -20, -20)$ & $q=1,2,4,8$ \\
CONSTR & 2 & 2 & 2 & $(-10, -60)$ & -- \\
Constrained Branin-Currin & 2 & 2 & 1 & $(-300, -20)$ & -- \\
DiscBrake & 4 & 2 & 4 & $(-10, -10)$ & $q=1,2,4,8$ \\
DTLZ2 & 5 & 4 & 0 & $(-10, -10, -10, -10)$ & -- \\
DTLZ3 & 3 & 2 & 0 & $(-100, -100)$ & -- \\
DTLZ4 & 5 & 4 & 0 & $(-1, -1, -1, -1)$ & -- \\
DTLZ5 & 5 & 4 & 0 & $(-1, -1, -1, -1)$ & -- \\
DTLZ7 & 20 & 2 & 0 & $(-25, -25)$ & -- \\
GMM & 2 & 4 & 0 & $(-1, -1, -1, -1)$ & -- \\
MW7 & 2 & 2 & 2 & $(-10, -10)$ & -- \\
OSY & 6 & 2 & 6 & $(-100, -100)$ & $q=1,2,4,8$ \\
Penicillin & 7 & 3 & 0 & $(-40, -70, -1000)$ & -- \\
SRN & 2 & 2 & 2 & $(-1000, -200)$ & -- \\
VehicleSafety & 5 & 3 & 0 & $(-1900, -20, -10)$ & -- \\
ZDT1 & 5 & 2 & 0 & $(-1, -1)$ & -- \\
ZDT2 & 5 & 2 & 0 & $(-1, -1)$ & -- \\
\bottomrule
\end{tabularx}
\end{table}

\begin{figure}[htbp]
    \centering
    \includegraphics[width=\linewidth]{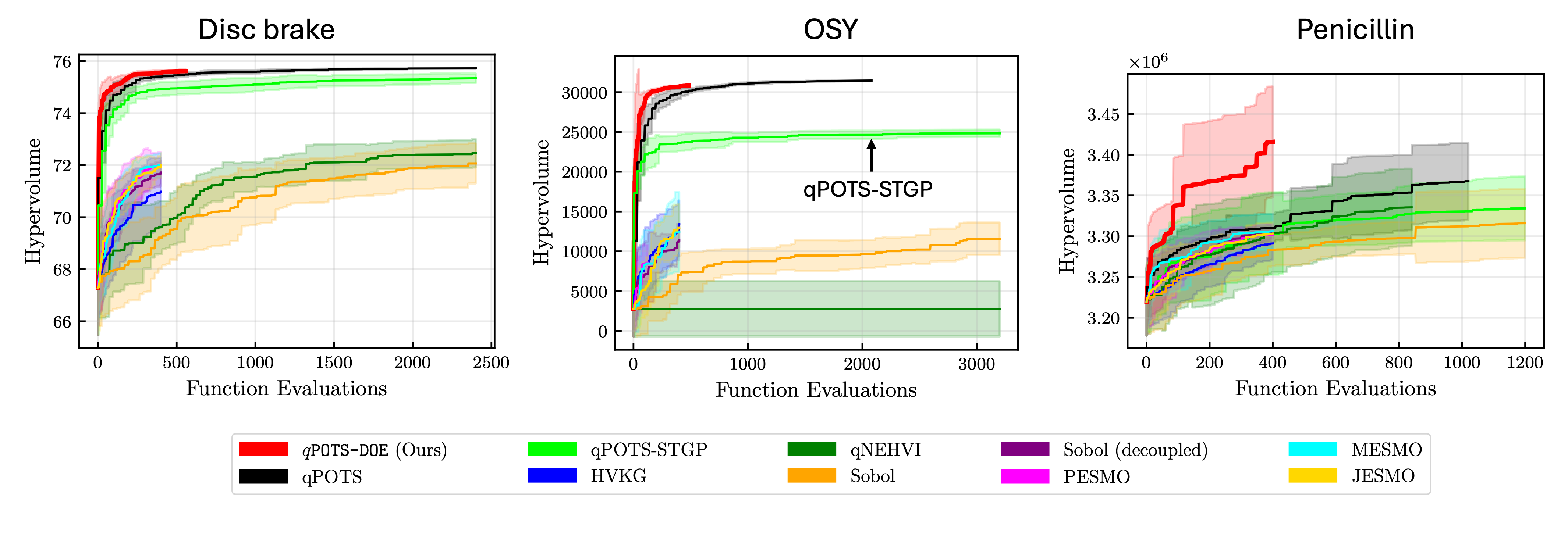}
    \caption{Coupled single-task GP ablation on Disc brake, OSY, and Penicillin.}
    \label{fig:ablation_stgp_coupled}
\end{figure}

\begin{figure}[htbp]
    \centering
    \includegraphics[width=\linewidth]{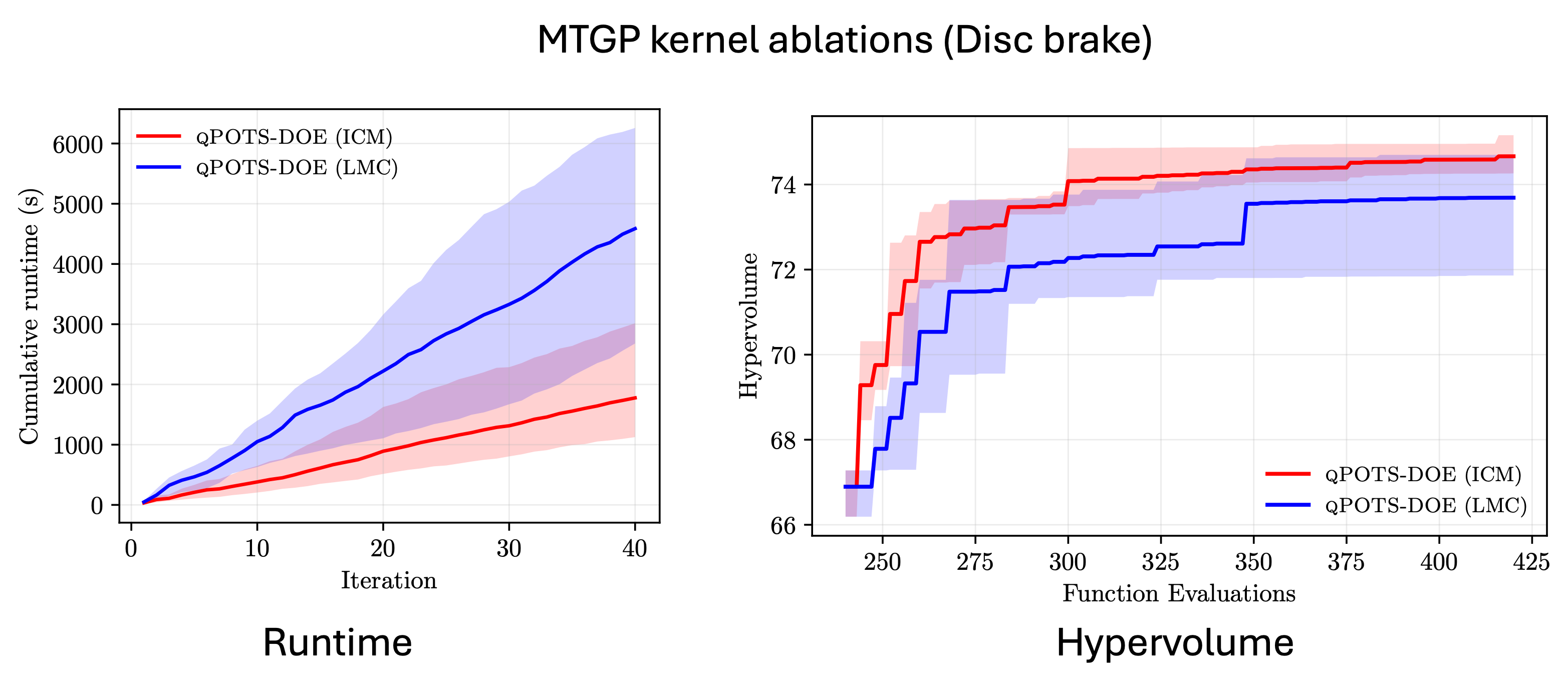}
    \caption{ICM versus LMC: cumulative runtime and hypervolume.}
    \label{fig:ablation_lmc_icm}
\end{figure}

\begin{figure}[htbp]
    \centering
    \includegraphics[width=\linewidth]{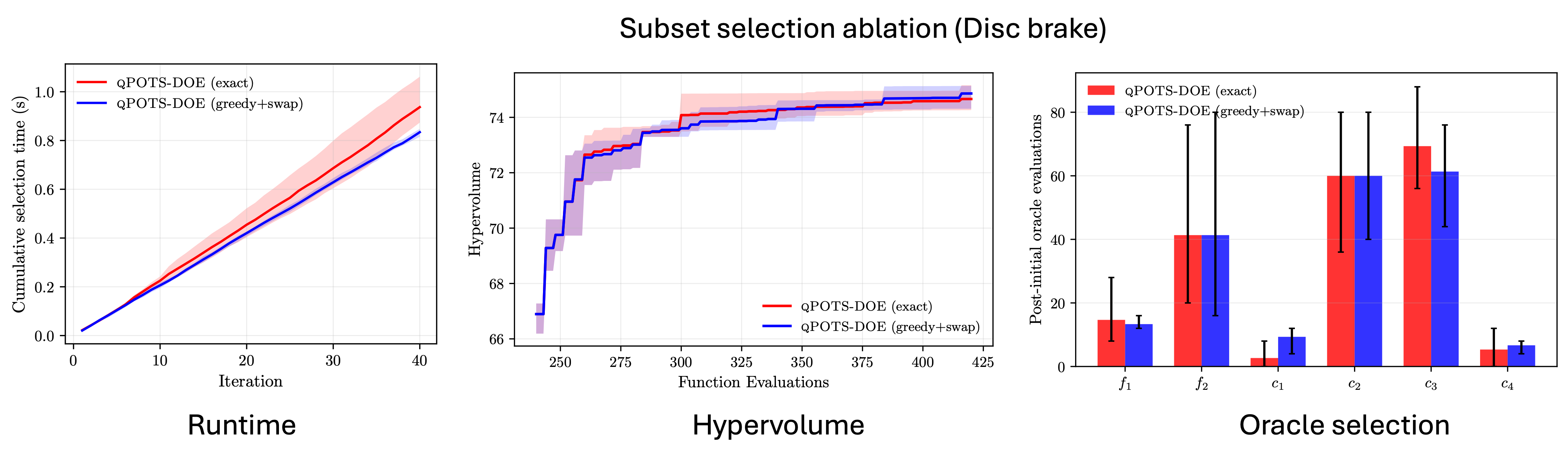}
    \caption{Exact versus greedy-plus-swap subset selection: cumulative selection time, hypervolume, and oracle evaluations.}
    \label{fig:ablation_subset_selection}
\end{figure}

\begin{figure}[htbp]
    \centering
    \includegraphics[width=\linewidth]{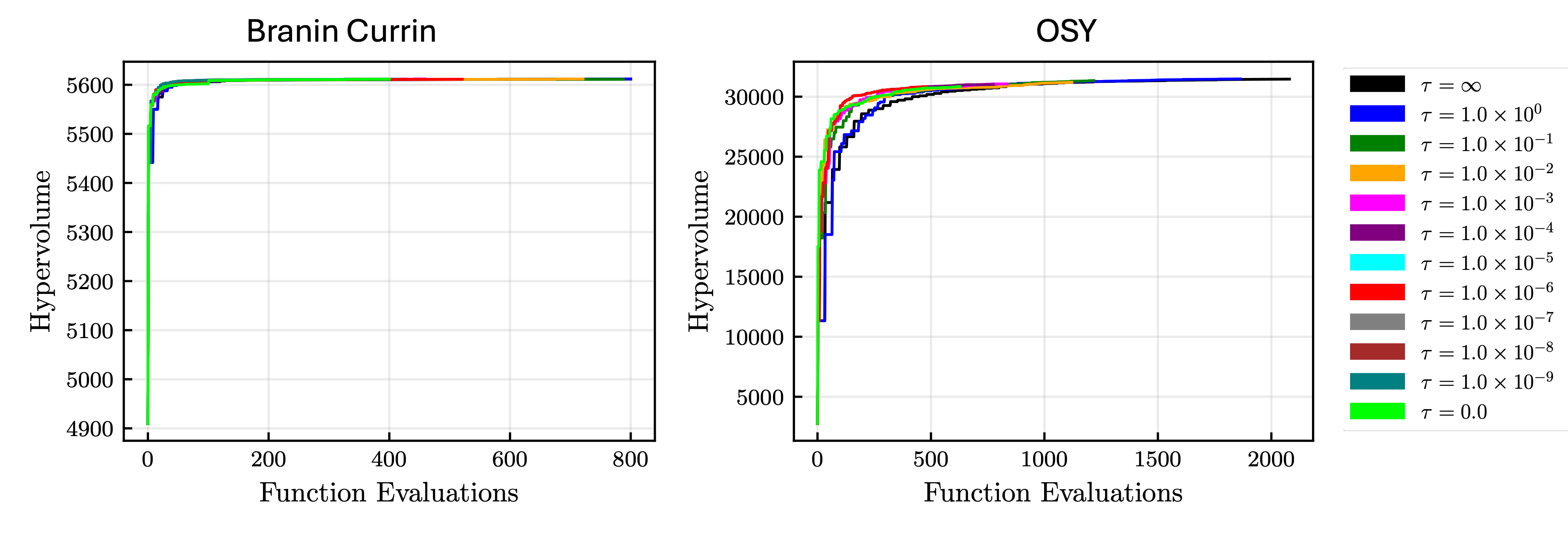}
    \caption{Ablation on the decoupling threshold $\tau$ for Branin-Currin and OSY.}
    \label{fig:ablation_tau}
\end{figure}

\begin{figure}[htbp]
    \centering
    \includegraphics[width=\linewidth]{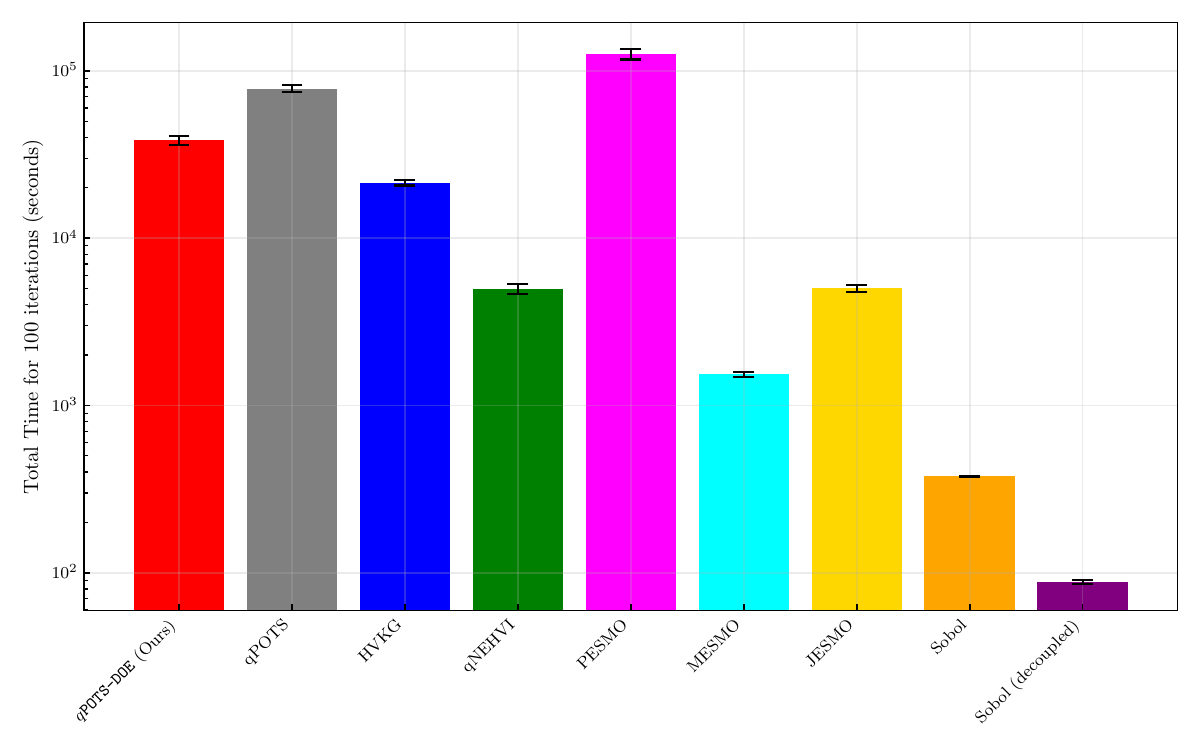}
    \caption{Runtime comparison on OSY.}
    \label{fig:ablation_osy_runtime}
\end{figure}

\clearpage

\subsection{Additional detail on the method}
\paragraph{Gaussian process models.}
\label{sec:gp_intro}
Let $f: \mathcal{X} \rightarrow \mbb{R}$ be an expensive zeroth-order oracle, and the input $\x$  be a $d$-dimensional vector $\mathbf{x}\in \mcl{X} \subset \mathbb{R}^d$, and let $Y(\mathbf{x})$ denote a real-valued random function evaluated at $\mathbf{x}$. A GP is a stochastic process such that any finite collection of function values is jointly Gaussian \cite{rasmussen2006gaussian,williams1998prediction}. We write the GP prior as
\begin{equation*}
  Y(\mathbf{x}) \sim \mathcal{GP}\!\big(m(\mathbf{x}),\, k(\mathbf{x},\mathbf{x}')\big),
\end{equation*}
where $m(\mathbf{x})=\mathbb{E}[Y(\mathbf{x})]$ is the mean function and $k(\mathbf{x},\mathbf{x}')=\mathrm{cov}(Y(\mathbf{x}),Y(\mathbf{x}'))$ is a positive semidefinite covariance (kernel) function.
Given a set of inputs $\mathbf{X} = [\mathbf{x}_1,\ldots,\mathbf{x}_n]^\top \in \mathbb{R}^{n\times d}$, define the vector of latent function values $\mathbf{y} \;=\; \big[Y(\mathbf{x}_1),\,\ldots,\,Y(\mathbf{x}_n)\big]^\top \in \mathbb{R}^n.$
By the definition of a GP, $\mathbf{y}$ is multivariate Gaussian: $ \mathbf{y} \sim \mathcal{N}(\mathbf{m}, \Sigma),
  \quad
  \mathbf{m}_i = m(\mathbf{x}_i),
  \quad
  \Sigma_{ij} = k(\mathbf{x}_i,\mathbf{x}_j).$
Let there be noisy observations of $f$ of the form $ \tilde{y}_i = Y(\mathbf{x}_i) + \varepsilon_i,
  \qquad
  \varepsilon_i \overset{\text{i.i.d.}}{\sim} \mathcal{N}(0,\sigma^2),$
so that the observation vector $\tilde{\mathbf{y}}=[\tilde{y}_1,\ldots,\tilde{y}_n]^\top$ satisfies $  \tilde{\mathbf{y}} \sim \mathcal{N}\big(\mathbf{m},\, \Sigma + \sigma^2 \mathbf{I}_n\big).$
Then, for a test input $\mathbf{x}_*\in\mathbb{R}^d$, define the cross-covariances $\Sigma_{X*} = \big[k(\mathbf{x}_i,\mathbf{x}_*)\big]_{i=1}^n \in \mathbb{R}^{n\times 1},
  \quad
  \Sigma_{*X} = \Sigma_{X*}^\top,
  \quad
  \Sigma_{**} = k(\mathbf{x}_*,\mathbf{x}_*) \in \mathbb{R}.$
Conditioning the joint Gaussian of $(\tilde{\mathbf{y}},\,Y(\mathbf{x}_*))$ yields the standard GP regression posterior \cite{rasmussen2006gaussian}:
\begin{align*}
  Y(\mathbf{x}_*) \mid \mathbf{X}, \tilde{\mathbf{y}}
  \;\sim\;
  \mathcal{N}\!\Big(
    &m(\mathbf{x}_*) + \Sigma_{*X}\big(\Sigma + \sigma^2\mathbf{I}_n\big)^{-1}\big(\tilde{\mathbf{y}}-\mathbf{m}\big),
    \nonumber
    \Sigma_{**} - \Sigma_{*X}\big(\Sigma + \sigma^2\mathbf{I}_n\big)^{-1}\Sigma_{X*}
  \Big).
  \label{eq:gp_posterior}
\end{align*}
This posterior provides both a predictive mean and a calibrated uncertainty (predictive variance) at $\mathbf{x}_*$.

\subsubsection{Unconstrained \qpotsdoe~with MTGPs}
\label{sec:unconstrained_qpots}

Let
$
\mathbf{Y}_{1:K}(\x)
\equiv
\begin{bmatrix}
Y_1(\x) & \cdots & Y_K(\x)
\end{bmatrix}^{\top}.
$
Then, the multitask posterior is written as
\[
\mathbf{Y}_{1:K}(\cdot)\mid D_n^{1:K},\Omega
\sim
\mathrm{MTGP}\!\left(
\boldsymbol{\mu}_n(\cdot),
\boldsymbol{\Sigma}_n(\cdot,\cdot)
\right),
\]
where
\[
\left[\boldsymbol{\Sigma}_n(\x,\x')\right]_{ij}
=
\operatorname{cov}
\left(
Y_i(\x),Y_j(\x') \mid D_n^{1:K},\Omega
\right),
\qquad i,j\in[K].
\]
Thus, the posterior over the $K$ objective sample paths is no longer
factorized across objectives. We choose $\x$ according to
\begin{equation}
    p_{X^\ast}(\x)
=
\int
\delta
\left(
\x-
\operatorname*{arg\,max}_{\x\in\mathcal{X}}
\{Y_1(\x),\ldots,Y_K(\x)\}
\right)
p\!\left(\mathbf{Y}_{1:K}\mid D_n^{1:K}\right)
\,d\mathbf{Y}_{1:K},
\label{eqn:qpots}
\end{equation}
where $d\mathbf{Y}_{1:K}
\equiv
dY_1\cdots dY_K$.
If $\boldsymbol{\Sigma}_n(\cdot,\cdot)$ is block diagonal across the
$K$ tasks, then \eqref{eqn:qpots} reduces to the independent-GP factorization as in \cite{renganathan2025qpots}.
In practice, this is equivalent to drawing one joint sample path from the
multitask posterior,
\[
\mathbf{Y}_{1:K}(\cdot,\omega)
=
\boldsymbol{\mu}_n(\cdot)
+
\boldsymbol{\Sigma}_n^{1/2}(\cdot)\,Z(\omega),
\qquad
Z(\omega)\sim\mathcal{N}(0,I),
\]
and choosing the next point(s) from the Pareto set of the jointly sampled
objective paths:
\[
X^\ast
=
\operatorname*{arg\,max}_{\x\in\mathcal{X}}
\{Y_1(\x,\omega),\ldots,Y_K(\x,\omega)\}.
\]
As in \cite{renganathan2025qpots}, this cheap multiobjective optimization problem is solved via evolutionary approaches e.g., NSGA-II. 

Let
$\gamma(\cdot,\cdot):\mathbb{R}^d\times\mathbb{R}^d\rightarrow
\mathbb{R}_{+}$ denote the Euclidean distance between two points in
$\mathcal{X}$. Then, the next batch of $q$ points is chosen per
\begin{equation}
    \begin{aligned}
\x_{n+1}
&=
\operatorname*{arg\,max}_{\x^\ast\in \X^\ast}
\min_{\x_i\in X_n}
\gamma(\x^\ast,\x_i), \\
\x_{n+2}
&=
\operatorname*{arg\,max}_{\x^\ast\in X^\ast}
\min_{\x_i\in X_n\cup\{\x_{n+1}\}}
\gamma(\x^\ast,\x_i), \\
&\hspace{0.25in}\vdots \\
\x_{n+q}
&=
\operatorname*{arg\,max}_{\x^\ast\in X^\ast}
\min_{\x_i\in X_n\cup\{\x_{n+1},\ldots,\x_{n+q-1}\}}
\gamma(\x^\ast,\x_i).
\end{aligned}
\label{eqn:maximin}
\end{equation}
\Cref{eqn:maximin} is unchanged from the independent-GP setting; only the
posterior sample path used to construct $X^\ast$ is now drawn jointly
from the multitask GP posterior.

\subsection{Proofs}
\label{sec:proofs}

\begin{assumption}[Regularity and no-starvation]
\label{ass:regularity-no-starvation}
Let $\mathcal X \subset \mathbb R^d$ be compact with nonempty interior, and let
\[
\mathbf Y(\mathbf x)
=
\big(Y_1(\mathbf x),\ldots,Y_K(\mathbf x)\big)^\top
\]
be a correctly specified $K$-task Gaussian process prior with mean
$\boldsymbol\mu_0(\mathbf x)$ and matrix-valued covariance kernel
\[
\mathbf K_0(\mathbf x,\mathbf x')
=
\operatorname{Cov}\!\left(
\mathbf Y(\mathbf x),\mathbf Y(\mathbf x')
\right)
\in \mathbb R^{K\times K}.
\]
Assume that the process admits an almost surely continuous version on
$\mathcal X$.
At round $i$, the algorithm queries a design point $\mathbf x_i\in\mathcal X$
and observes only a subset of tasks $S_i\subseteq [K]$. For each
$k\in S_i$, the observation model is
\[
y_{i,k}
=
Y_k(\mathbf x_i)+\varepsilon_{i,k},
\qquad
\varepsilon_{i,k}\sim \mathcal N(0,\sigma_k^2),
\]
where the noise variables are independent across rounds and tasks, with
$0<\sigma_k^2<\infty$. Let
\[
\mathcal D_n
=
\{(\mathbf x_i,S_i,\{y_{i,k}:k\in S_i\})\}_{i=1}^n
\]
denote the partial-observation data after $n$ rounds.

Finally, assume no starvation under decoupling: for every task
$k\in[K]$, the set of design points at which task $k$ is directly observed,
\[
\mathcal X_{n,k}
:=
\{\mathbf x_i : i\le n,\ k\in S_i\},
\]
has fill distance
\[
h_{n,k}
:=
\sup_{\mathbf x\in\mathcal X}
\min_{\mathbf z\in\mathcal X_{n,k}}
\|\mathbf x-\mathbf z\|_2
\]
satisfying
\[
h_{n,k}\longrightarrow 0
\qquad\text{as } n\to\infty,
\qquad
\forall k\in[K].
\]
Note that the no-starvation assumption is only necessary to prove \Cref{thm:qpots-doe-consistency}; but, in practice, it may not be necessary. Indeed, all experiments reported in this paper don't perform line $16$ of \Cref{alg:qpots-doe} (occasional evaluations of all objectives and constraints).
\end{assumption}

\begin{theorem}[Posterior consistency under decoupled oracle evaluations]
\label{thm:qpots-doe-consistency}
Let $\mathcal X \subset \mathbb R^d$ be compact with nonempty interior, and let
\[
\mathbf Y(\mathbf x)
=
\big(Y_1(\mathbf x),\ldots,Y_K(\mathbf x)\big)^\top
\]
denote a $K$-task Gaussian process prior with mean
$\boldsymbol\mu_0:\mathcal X\to\mathbb R^K$ and matrix-valued covariance kernel
\[
\mathbf K_0(\mathbf x,\mathbf x')
=
\operatorname{Cov}\!\left(
\mathbf Y(\mathbf x),\mathbf Y(\mathbf x')
\right)
\in \mathbb R^{K\times K}.
\]
Assume that the latent objective vector
\[
\mathbf Y_0(\mathbf x)
=
\big(Y_{0,1}(\mathbf x),\ldots,Y_{0,K}(\mathbf x)\big)^\top
\]
is generated from this correctly specified $K$-task GP model, and that the
process admits an almost surely continuous version on $\mathcal X$.

At round $i$, the algorithm selects a design point $\mathbf x_i\in\mathcal X$
and evaluates only a subset of objectives
\[
S_i \subseteq [K],
\qquad
\bar S_i := [K]\setminus S_i .
\]
For every $k\in S_i$, the observation model is
\[
y_{i,k}
=
Y_{0,k}(\mathbf x_i)+\varepsilon_{i,k},
\qquad
\varepsilon_{i,k}\sim \mathcal N(0,\sigma_k^2),
\]
where the noise variables are independent across rounds and objectives, with
$0<\sigma_k^2<\infty$. No observation is collected for objectives
$k\in\bar S_i$. Let
\[
\mathcal D_n
=
\left\{
\left(\mathbf x_i,S_i,\{y_{i,k}:k\in S_i\}\right)
\right\}_{i=1}^n
\]
denote the partial-observation data after $n$ rounds, and write the resulting
multitask GP posterior as
\[
\mathbf Y \mid \mathcal D_n
\sim
\mathcal{GP}\!\left(
\boldsymbol\mu_n(\cdot),
\mathbf K_n(\cdot,\cdot)
\right).
\]

Assume no starvation under decoupling: for every objective $k\in[K]$, the set
of design points at which objective $k$ is directly evaluated,
\[
\mathcal X_{n,k}
:=
\{\mathbf x_i : i\le n,\ k\in S_i\},
\]
has fill distance
\[
h_{n,k}
:=
\sup_{\mathbf x\in\mathcal X}
\min_{\mathbf z\in\mathcal X_{n,k}}
\|\mathbf x-\mathbf z\|_2
\]
satisfying
\[
h_{n,k}\longrightarrow 0
\qquad
\text{as } n\to\infty,
\qquad
\forall k\in[K].
\]

Then, for any finite set
\[
\mathcal X^\star
=
\{\mathbf x^{(1)},\ldots,\mathbf x^{(m)}\}
\subset \mathcal X
\]
and any $\epsilon>0$,
\[
\Pi\!\left(
\left\|
\mathbf Y(\mathcal X^\star)
-
\mathbf Y_0(\mathcal X^\star)
\right\|_2
>
\epsilon
\;\middle|\;
\mathcal D_n
\right)
\longrightarrow 0
\qquad
\text{almost surely under the true data-generating law}.
\]
Equivalently,
\[
\boldsymbol\mu_n(\mathcal X^\star)
\longrightarrow
\mathbf Y_0(\mathcal X^\star),
\qquad
\mathbf K_n(\mathcal X^\star,\mathcal X^\star)
\longrightarrow
\mathbf 0
\]
almost surely. In particular, for every fixed $\mathbf x\in\mathcal X$,
\[
\boldsymbol\mu_n(\mathbf x)
\longrightarrow
\mathbf Y_0(\mathbf x),
\qquad
\mathbf K_n(\mathbf x,\mathbf x)
\longrightarrow
\mathbf 0_{K\times K}.
\]
\end{theorem}

\begin{proof}
Let $\mathcal F_n := \sigma(\mathcal D_n)$ denote the filtration generated by
the partial-observation data. For each task $k\in[K]$, define the task-specific
data
\[
\mathcal D_{n,k}
:=
\{(\mathbf x_i,y_{i,k}) : i\le n,\ k\in S_i\}.
\]
Thus $\mathcal D_{n,k}$ contains only the direct observations of objective $k$,
whereas $\mathcal D_n$ contains all observations from the evaluated subsets
$S_i$ and no observations from the omitted subsets
$\bar S_i=[K]\setminus S_i$.

We first show that the posterior variance of each task vanishes at every fixed
input. Fix $k\in[K]$ and $\mathbf x\in\mathcal X$. Consider the scalar GP
regression problem for $Y_k$ using only $\mathcal D_{n,k}$, and write
\[
\sigma_{n,k}^2(\mathbf x)
:=
\operatorname{Var}\!\left(Y_k(\mathbf x)\mid \mathcal D_{n,k}\right).
\]
Let $\mu_k$ denote the prior mean of $Y_k$, and define the centered process
$\widetilde Y_k(\mathbf x):=Y_k(\mathbf x)-\mu_k(\mathbf x)$. Since the covariance
kernel is continuous and the process admits an almost surely continuous version,
$Y_k$ is mean-square continuous. Hence
\[
\omega_k(r)
:=
\sup_{\substack{\mathbf u,\mathbf v\in\mathcal X\\
\|\mathbf u-\mathbf v\|_2\le r}}
\left(
\mathbb E\left[
\big(\widetilde Y_k(\mathbf u)-\widetilde Y_k(\mathbf v)\big)^2
\right]
\right)^{1/2}
\longrightarrow 0
\qquad
\text{as } r\downarrow 0 .
\]

By the no-starvation condition, the task-$k$ design set
\[
\mathcal X_{n,k}:=\{\mathbf x_i:i\le n,\ k\in S_i\}
\]
has fill distance $h_{n,k}\to 0$. Therefore, for every fixed
$\mathbf x\in\mathcal X$, there exists a sequence $r_n\downarrow 0$ such that
the number of task-$k$ observations in the local ball
\[
I_{n,k}(\mathbf x)
:=
\{i\le n:k\in S_i,\ \|\mathbf x_i-\mathbf x\|_2\le r_n\}
\]
satisfies
\[
N_{n,k}(\mathbf x):=|I_{n,k}(\mathbf x)|\longrightarrow\infty .
\]
For example, on a compact box or any compact domain with the usual local
covering regularity, this follows from the fill-distance condition by taking
$r_n/h_{n,k}\to\infty$ and $r_n\to 0$.

Define the local linear estimator
\[
\widehat Y_{n,k}(\mathbf x)
:=
\mu_k(\mathbf x)
+
\frac{1}{N_{n,k}(\mathbf x)}
\sum_{i\in I_{n,k}(\mathbf x)}
\big(y_{i,k}-\mu_k(\mathbf x_i)\big).
\]
Since
\[
y_{i,k}
=
Y_k(\mathbf x_i)+\varepsilon_{i,k},
\qquad
\varepsilon_{i,k}\sim \mathcal N(0,\sigma_k^2),
\]
with independent noise, we have
\[
Y_k(\mathbf x)-\widehat Y_{n,k}(\mathbf x)
=
\frac{1}{N_{n,k}(\mathbf x)}
\sum_{i\in I_{n,k}(\mathbf x)}
\big(
\widetilde Y_k(\mathbf x)-\widetilde Y_k(\mathbf x_i)
\big)
-
\frac{1}{N_{n,k}(\mathbf x)}
\sum_{i\in I_{n,k}(\mathbf x)}
\varepsilon_{i,k}.
\]
Using Cauchy's inequality in $L^2$ and the independence of the observation
noise,
\[
\mathbb E\!\left[
\big(Y_k(\mathbf x)-\widehat Y_{n,k}(\mathbf x)\big)^2
\right]
\le
\omega_k(r_n)^2
+
\frac{\sigma_k^2}{N_{n,k}(\mathbf x)}
\longrightarrow 0 .
\]
The scalar GP posterior variance is the minimum mean-squared error among
estimators measurable with respect to $\mathcal D_{n,k}$. Therefore,
\[
\sigma_{n,k}^2(\mathbf x)
=
\operatorname{Var}\!\left(Y_k(\mathbf x)\mid \mathcal D_{n,k}\right)
\le
\mathbb E\!\left[
\big(Y_k(\mathbf x)-\widehat Y_{n,k}(\mathbf x)\big)^2
\right]
\longrightarrow 0 .
\]

Now compare the scalar posterior to the full multitask posterior. Since
$\mathcal D_{n,k}\subseteq\mathcal D_n$, conditioning on the full multitask data
cannot increase conditional variance. Hence
\[
\big[\mathbf K_n(\mathbf x,\mathbf x)\big]_{kk}
=
\operatorname{Var}\!\left(Y_k(\mathbf x)\mid \mathcal D_n\right)
\le
\operatorname{Var}\!\left(Y_k(\mathbf x)\mid \mathcal D_{n,k}\right)
=
\sigma_{n,k}^2(\mathbf x)
\longrightarrow 0 .
\]
This holds for every $k\in[K]$. Since
$\mathbf K_n(\mathbf x,\mathbf x)$ is positive semidefinite,
\[
\left|
\big[\mathbf K_n(\mathbf x,\mathbf x)\big]_{k\ell}
\right|
\le
\sqrt{
\big[\mathbf K_n(\mathbf x,\mathbf x)\big]_{kk}
\big[\mathbf K_n(\mathbf x,\mathbf x)\big]_{\ell\ell}
},
\]
so every off-diagonal entry also converges to zero. Therefore,
\[
\mathbf K_n(\mathbf x,\mathbf x)\longrightarrow \mathbf 0_{K\times K}
\qquad
\text{for every fixed } \mathbf x\in\mathcal X .
\]

Next, let
\[
\mathcal X^\star
=
\{\mathbf x^{(1)},\ldots,\mathbf x^{(m)}\}
\subset\mathcal X
\]
be any finite set, and define the stacked latent vector
\[
\mathbf Z
:=
\mathbf Y(\mathcal X^\star)
:=
\big(
\mathbf Y(\mathbf x^{(1)})^\top,\ldots,
\mathbf Y(\mathbf x^{(m)})^\top
\big)^\top
\in\mathbb R^{mK}.
\]
Under the multitask GP posterior,
\[
\mathbf Z\mid\mathcal D_n
\sim
\mathcal N(\boldsymbol\mu_n^\star,\mathbf K_n^\star),
\]
where
\[
\boldsymbol\mu_n^\star
:=
\mathbb E[\mathbf Z\mid\mathcal F_n],
\qquad
\mathbf K_n^\star
:=
\operatorname{Cov}(\mathbf Z\mid\mathcal F_n).
\]
Applying the previous variance argument to each component
$Y_k(\mathbf x^{(j)})$ gives
\[
\operatorname{tr}(\mathbf K_n^\star)\longrightarrow 0 .
\]
Since $\mathbf K_n^\star$ is positive semidefinite, this also implies
\[
\mathbf K_n^\star\longrightarrow \mathbf 0_{mK\times mK}.
\]

It remains to show that the posterior mean converges to the true latent values.
By correct specification, the true objective vector $\mathbf Y_0(\mathcal X^\star)$
is the realized value of $\mathbf Z$ under the same GP data-generating law.
Moreover,
\[
\boldsymbol\mu_n^\star
=
\mathbb E[\mathbf Z\mid\mathcal F_n]
\]
is an $L^2$ martingale with respect to $(\mathcal F_n)_{n\ge 1}$. Since
\[
\mathbb E\!\left[
\|\mathbf Z-\boldsymbol\mu_n^\star\|_2^2
\mid \mathcal F_n
\right]
=
\operatorname{tr}(\mathbf K_n^\star)
\longrightarrow 0,
\]
and the conditional variances are bounded above by the prior second moment,
dominated convergence gives
\[
\mathbb E\!\left[
\|\mathbf Z-\boldsymbol\mu_n^\star\|_2^2
\right]
\longrightarrow 0 .
\]
The martingale convergence theorem implies that
$\boldsymbol\mu_n^\star$ converges almost surely and in $L^2$ to
$\mathbb E[\mathbf Z\mid\mathcal F_\infty]$, where
$\mathcal F_\infty:=\sigma(\cup_{n\ge 1}\mathcal F_n)$. Since the $L^2$ limit is
unique and $\boldsymbol\mu_n^\star\to \mathbf Z$ in $L^2$, we obtain
\[
\boldsymbol\mu_n^\star
\longrightarrow
\mathbf Z
=
\mathbf Y_0(\mathcal X^\star)
\qquad
\text{almost surely}.
\]

Finally, for any $\epsilon>0$,
\[
\begin{aligned}
&\Pi\!\left(
\left\|
\mathbf Y(\mathcal X^\star)
-
\mathbf Y_0(\mathcal X^\star)
\right\|_2
>
\epsilon
\;\middle|\;
\mathcal D_n
\right)
\\
&\qquad\le
\Pi\!\left(
\left\|
\mathbf Y(\mathcal X^\star)-\boldsymbol\mu_n^\star
\right\|_2
>
\epsilon/2
\;\middle|\;
\mathcal D_n
\right)
+
\mathbf 1\!\left\{
\left\|
\boldsymbol\mu_n^\star
-
\mathbf Y_0(\mathcal X^\star)
\right\|_2
>
\epsilon/2
\right\}.
\end{aligned}
\]
By Chebyshev's inequality,
\[
\Pi\!\left(
\left\|
\mathbf Y(\mathcal X^\star)-\boldsymbol\mu_n^\star
\right\|_2
>
\epsilon/2
\;\middle|\;
\mathcal D_n
\right)
\le
\frac{4\,\operatorname{tr}(\mathbf K_n^\star)}{\epsilon^2}
\longrightarrow 0,
\]
and the indicator term converges to zero almost surely because
$\boldsymbol\mu_n^\star\to \mathbf Y_0(\mathcal X^\star)$ almost surely.
Therefore,
\[
\Pi\!\left(
\left\|
\mathbf Y(\mathcal X^\star)
-
\mathbf Y_0(\mathcal X^\star)
\right\|_2
>
\epsilon
\;\middle|\;
\mathcal D_n
\right)
\longrightarrow 0
\qquad
\text{almost surely}.
\]
Since $\mathcal X^\star$ was arbitrary, the same argument applied to a singleton
$\mathcal X^\star=\{\mathbf x\}$ gives, for every fixed
$\mathbf x\in\mathcal X$,
\[
\boldsymbol\mu_n(\mathbf x)\longrightarrow \mathbf Y_0(\mathbf x),
\qquad
\mathbf K_n(\mathbf x,\mathbf x)\longrightarrow \mathbf 0_{K\times K}.
\]
This proves posterior consistency under decoupled oracle evaluations.
\end{proof}

\begin{theorem}[Finite-sample information transfer under decoupling]
\label{thm:finite_sample_information_transfer}
Fix a round \(n\), a candidate design \(\mathbf{x}\in \mathcal X\), and suppose that the
\(K\)-oracle multitask posterior at \(\mathbf{x}\) satisfies
\[
Y_{1:K}(\mathbf{x}) \mid D_n
\sim
\mathcal N\!\left(\mu_n(\mathbf{x}),\Sigma_n(\mathbf{x})\right),
\qquad
\Sigma_n(\mathbf{x}) \succ 0 .
\]
Let \(\mathcal K \triangleq \{1,\ldots,K\}\). For any nonempty proper subset
\(S\subsetneq \mathcal K\), write
\[
\bar S \triangleq \mathcal K \setminus S .
\]
Suppose that a decoupled evaluation observes the selected oracle values through
\[
Z_S(\mathbf{x})
=
Y_S(\mathbf{x})+\varepsilon_S,
\qquad
\varepsilon_S\sim \mathcal N(0,\Sigma_{\varepsilon,SS}),
\]
where \(\varepsilon_S\) is independent of \(Y_{1:K}(\mathbf{x})\mid D_n\) and
\(\Sigma_{\varepsilon,SS}\succeq 0\). Define the expected one-step posterior entropy
reduction about the unevaluated oracles \(Y_{\bar S}(\mathbf{x})\) by
\[
\Delta_n(S;\mathbf{x})
\triangleq
H\!\left(Y_{\bar S}(\mathbf{x})\mid D_n\right)
-
\mathbb E_{Z_S(\mathbf{x})\mid D_n}
\left[
H\!\left(Y_{\bar S}(\mathbf{x})\mid D_n,Z_S(\mathbf{x})\right)
\right].
\]
Then
\[
\Delta_n(S;\mathbf{x})
=
I\!\left(Z_S(\mathbf{x});Y_{\bar S}(\mathbf{x})\mid D_n\right),
\]
and this quantity has the closed form
\[
\Delta_n(S;\mathbf{x})
=
\frac{1}{2}
\log
\frac{
\left|\Sigma_{n,\bar S\bar S}(\mathbf{x})\right|
}{
\left|
\Sigma_{n,\bar S\bar S}(\mathbf{x})
-
\Sigma_{n,\bar S S}(\mathbf{x})
\left(\Sigma_{n,SS}(\mathbf{x})+\Sigma_{\varepsilon,SS}\right)^{-1}
\Sigma_{n,S\bar S}(\mathbf{x})
\right|
}.
\]
Equivalently,
\[
\Delta_n(S;\mathbf{x})
=
\frac{1}{2}
\log
\frac{
\left|\Sigma_{n,SS}(\mathbf{x})+\Sigma_{\varepsilon,SS}\right|
\left|\Sigma_{n,\bar S\bar S}(\mathbf{x})\right|
}{
\left|
\begin{bmatrix}
\Sigma_{n,SS}(\mathbf{x})+\Sigma_{\varepsilon,SS}
&
\Sigma_{n,S\bar S}(\mathbf{x})
\\
\Sigma_{n,\bar S S}(\mathbf{x})
&
\Sigma_{n,\bar S\bar S}(\mathbf{x})
\end{bmatrix}
\right|
}.
\]
Consequently, for any fixed decoupled oracle budget \(M\in\{1,\ldots,K-1\}\), the subset
\[
S_n^\star(\mathbf{x})
\in
\arg\max_{S\subseteq \mathcal K:\ |S|=M}
I\!\left(Z_S(\mathbf{x});Y_{\bar S}(\mathbf{x})\mid D_n\right)
\]
maximizes the expected one-step posterior entropy reduction about the unevaluated oracles.

In the noiseless latent-oracle case \(\Sigma_{\varepsilon,SS}=0\), this reduces to
\[
\Delta_n(S;\mathbf{x})
=
I\!\left(Y_S(\mathbf{x});Y_{\bar S}(\mathbf{x})\mid D_n\right)
=
\frac{1}{2}
\log
\frac{
\left|\Sigma_{n,SS}(\mathbf{x})\right|
\left|\Sigma_{n,\bar S\bar S}(\mathbf{x})\right|
}{
\left|\Sigma_n(\mathbf{x})\right|
},
\]
which is exactly the Gaussian mutual-information criterion used for subset selection.

Moreover, let
\[
\mathrm{TC}_n(\mathbf{x})
\triangleq
\sum_{k=1}^K
H\!\left(Y_k(\mathbf{x})\mid D_n\right)
-
H\!\left(Y_{1:K}(\mathbf{x})\mid D_n\right),
\]
and define \(\mathrm{TC}_{n,S}(\mathbf{x})\) and \(\mathrm{TC}_{n,\bar S}(\mathbf{x})\) analogously for the
subvectors \(Y_S(\mathbf{x})\) and \(Y_{\bar S}(\mathbf{x})\). Then, for every
\(S\subsetneq \mathcal K\),
\[
0
\le
\Delta_n(S;\mathbf{x})
\le
I\!\left(Y_S(\mathbf{x});Y_{\bar S}(\mathbf{x})\mid D_n\right)
=
\mathrm{TC}_n(\mathbf{x})
-
\mathrm{TC}_{n,S}(\mathbf{x})
-
\mathrm{TC}_{n,\bar S}(\mathbf{x})
\le
\mathrm{TC}_n(\mathbf{x}).
\]
Thus, if \(\mathrm{TC}_n(\mathbf{x})<\tau\), then no decoupled \(M\)-oracle evaluation at
\(\mathbf{x}\) can reduce the posterior entropy of the unevaluated oracles by more than
\(\tau\) nats. Finally,
\[
\Delta_n(S;\mathbf{x})=0
\quad\Longleftrightarrow\quad
\Sigma_{n,S\bar S}(\mathbf{x})=0,
\]
so a decoupled evaluation transfers strictly positive information to the unevaluated oracles
exactly when the selected and unevaluated oracle blocks have nonzero posterior
cross-covariance at \(\mathbf{x}\).
\end{theorem}

\begin{proof}
Condition throughout on \(\mcl D_n\). Since
\[
\begin{bmatrix}
Z_S(\mathbf{x})\\
Y_{\bar S}(\mathbf{x})
\end{bmatrix}
\]
is jointly Gaussian, the conditional covariance of
\(Y_{\bar S}(\mathbf{x})\) after observing \(Z_S(\mathbf{x})\) is
\[
\Sigma_{n,\bar S\bar S}(\mathbf{x})
-
\Sigma_{n,\bar S S}(\mathbf{x})
\left(\Sigma_{n,SS}(\mathbf{x})+\Sigma_{\varepsilon,SS}\right)^{-1}
\Sigma_{n,S\bar S}(\mathbf{x}).
\]
This covariance does not depend on the realized value of \(Z_S(\mathbf{x})\). Applying the
Gaussian entropy formula
\[
H(W)=\frac{1}{2}\log\!\left((2\pi e)^m |\Sigma_W|\right)
\]
therefore gives the stated expression for \(\Delta_n(S;\mathbf{x})\). Since expected entropy
reduction is mutual information,
\[
\Delta_n(S;\mathbf{x})
=
I\!\left(Z_S(\mathbf{x});Y_{\bar S}(\mathbf{x})\mid \mcl D_n\right).
\]
Maximizing this quantity over all subsets \(S\subseteq\mathcal K\) with \(|S|=M\) therefore
maximizes the expected one-step posterior entropy reduction about the unevaluated oracles.

For the total-correlation bound, the observation model gives the Markov relation
\[
Y_{\bar S}(\mathbf{x}) \longrightarrow Y_S(\mathbf{x}) \longrightarrow Z_S(\mathbf{x})
\]
conditional on \(\mcl D_n\). Hence, by the data-processing inequality,
\[
I\!\left(Z_S(\mathbf{x});Y_{\bar S}(\mathbf{x})\mid \mcl D_n\right)
\le
I\!\left(Y_S(\mathbf{x});Y_{\bar S}(\mathbf{x})\mid \mcl D_n\right).
\]
Using
\[
\mathrm{TC}_n(\mathbf{x})
=
\sum_{k=1}^K H\!\left(Y_k(\mathbf{x})\mid \mcl D_n\right)
-
H\!\left(Y_{1:K}(\mathbf{x})\mid \mcl D_n\right),
\]
and the analogous definitions for \(\mathrm{TC}_{n,S}(\mathbf{x})\) and
\(\mathrm{TC}_{n,\bar S}(\mathbf{x})\), direct cancellation yields
\[
I\!\left(Y_S(\mathbf{x});Y_{\bar S}(\mathbf{x})\mid \mcl D_n\right)
=
\mathrm{TC}_n(\mathbf{x})
-
\mathrm{TC}_{n,S}(\mathbf{x})
-
\mathrm{TC}_{n,\bar S}(\mathbf{x}).
\]
Nonnegativity of total correlation gives the final upper bound by \(\mathrm{TC}_n(\mathbf{x})\).
Finally, for jointly Gaussian random vectors, zero mutual information is equivalent to zero
cross-covariance. Since
\[
\operatorname{Cov}\!\left(Z_S(\mathbf{x}),Y_{\bar S}(\mathbf{x})\mid D_n\right)
=
\Sigma_{n,S\bar S}(\mathbf{x}),
\]
we have \(\Delta_n(S;\mathbf{x})=0\) if and only if
\(\Sigma_{n,S\bar S}(\mathbf{x})=0\).
\end{proof}

\end{document}